\documentclass[11pt]{article}

\usepackage{vitercik}
\usepackage{subcaption}

\newcommand{\A}{\mathbf{A}}
\title{Private Optimization Without Constraint Violations}
\author{
	Andr\'es Mu\~noz Medina, Umar Syed, Sergei Vassilvitskii \\
	Google Research\\
	\texttt{\{ammedina,usyed,sergeiv\}@google.com}
	\and 
	Ellen Vitercik \\ Carnegie Mellon University \\ \texttt{vitercik@cs.cmu.edu}}

\begin{document}
	
	\maketitle

\begin{abstract}
We study the problem of differentially private optimization with linear constraints when the right-hand-side of the constraints depends on private data. This type of problem appears in many applications, especially resource allocation.
Previous research provided solutions that retained privacy but sometimes violated the constraints.
In many settings, however, the constraints cannot be violated under any circumstances. To address this hard requirement, we present an algorithm that releases a nearly-optimal solution satisfying the constraints with probability 1. We also prove a lower bound demonstrating that the difference between the objective value of our algorithm's solution and the optimal solution is tight up to logarithmic factors among all differentially private algorithms. We conclude with experiments demonstrating that our algorithm can achieve nearly optimal performance while preserving privacy. 
\end{abstract}

\section{Introduction}\label{sec:intro}
Differential privacy~\citep{Dwork06:Calibrating} has emerged as the standard for reasoning about user privacy and private computations. A myriad of practical algorithms exist for a broad range of problems. We can now solve tasks in a private manner ranging from computing simple dataset statistics~\citep{Nissim:Smooth} to modern machine learning~\citep{Abadi16:Deep}. In this paper we add to this body of research by tackling a fundamental question of constrained optimization. 

Specifically, we study optimization problems with linear constraints and Lipschitz objective functions. This family of optimization problems includes linear programming and quadratic programming with linear constraints, which can be used to formulate diverse problems in computer science, as well as other fields such as engineering, manufacturing, and transportation. 
Resource allocation is an example of a common problem in this family: given multiple agents competing for limited goods, how should the goods be distributed among the agents? Whether assigning jobs to machines or partitioning network bandwidth among different applications, these problems have convex optimization formulations with linear constraints. Given that the input to these problems may come from private user data, it is imperative that we find solutions that do not leak information about any individual.  

Formally, the goal in linearly-constrained optimization is to find a vector $\vec{x}$ maximizing a function $g(\vec{x})$ subject to the constraint that $\A \vec{x} \leq \vec{b}$. Due in part to the breadth of problems covered by these approaches, the past several decades have seen the development of a variety of optimization algorithms with provable guarantees, as well as fast commercial solvers. The parameters $\A$ and $\vec{b}$ encode data about the specific problem instance at hand, and it is easy to come up with instances where simply releasing the optimal solution would leak information about this sensitive data.

As a concrete example, suppose there is a hospital with branches located throughout a state, each of which has a number of patients with a certain disease. A specific drug is required to treat the infected patients, which the hospital can obtain from a set of pharmacies. The goal is to determine which pharmacies should supply which hospital branches while minimizing the transportation cost. In Figure~\ref{fig:transportation}, we present this problem as a linear program (LP). The LP is defined by sensitive information: the constraint vector reveals the number of patients with the disease at each branch.
\begin{figure}
\[\begin{array}{lll}
     \text{minimize} & \sum_{i,j} c_{ij}x_{ij} &\\
     \text{such that}& \sum_{j = 1}^N x_{ij}  \leq s_i &\forall i \in [M]\\
     &\sum_{i = 1}^M x_{ij} \geq r_j &\forall j \in[N]\\
      &x_{ij} \geq 0 &\forall i  \in [M], \forall j \in [N].
\end{array}\]
\caption{The classic transportation problem formulated as a linear program. There are $N$ hospital branches and $M$ pharmacies. Each branch $j$ requires $r_j$ units of a specific drug. These values are sensitive because they reveal the number of people at each branch with a specific disease. Each pharmacy $i$ has a supply of $s_i$ units. It costs $c_{ij}$ dollars to transport a unit of the drug from pharmacy $i$ to hospital $j$. We use the notation $x_{ij}$ to denote the units of the drug transported from pharmacy $i$ to hospital $j$.}\label{fig:transportation}
\end{figure}

We provide tools with provable guarantees for solving linearly-constrained optimization problems while preserving \emph{differential privacy} (DP)~\citep{Dwork06:Calibrating}. Our algorithm applies to the setting where the constraint vector $\vec{b}$ depends on private data, as is the case in many resource allocation problems, such as the transportation problem above. This problem falls in the category of private optimization, for which there are multiple algorithms in the unconstrained case~\citep{Bassily14:ERM, Chaudhuri11:Differentially, KiferST12:PrivateReg}. To the best of our knowledge, only \citet{Hsu14:Privately} and \citet{Cummings15:Privacy} study differentially private linear programming---a special case of linearly-constrained optimization. Their algorithms are allowed to violate the constraints, which can be unacceptable in many applications. In our transportation example from Figure~\ref{fig:transportation}, if the constraints are violated, a hospital will not receive the drugs they require or a pharmacy will be asked to supply more drugs than they have in inventory. The importance of satisfying constraints motivates this paper's central question:

\smallskip
\emph{How can we privately solve optimization problems while ensuring that no constraint is violated?}

\subsection{Results overview}
Our goal is to privately solve optimization problems of the form $\max_{\vec{x} \in \R^n} \left\{g(\vec{x})  : \A\vec{x} \leq \vec{b}(D)\right\},$ where $g$ is $L$-Lipschitz and $\vec{b}(D) \in \R^m$ depends on a private database $D$. The database is a set of individuals' records, each of which is an element of a domain $\cX.$ 

To solve this problem, our $(\epsilon, \delta)$-differentially private algorithm maps $\vec{b}(D)$ to a nearby vector $\bar{\vec{b}}(D)$ and releases the vector maximizing $g(\vec{x})$ such that $\A \vec{x} \leq \bar{\vec{b}}(D)$. (We assume that $g$ can be optimized efficiently under linear constraints, which is the case, for example, when $g$ is convex.) We ensure that $\bar{\vec{b}}(D) \leq \vec{b}(D)$ coordinate-wise, and therefore our algorithm's output satisfies the constraints. This requirement precludes our use of traditional DP mechanisms: perturbing each component of $\vec{b}(D)$ using the Laplace, Gaussian, or exponential mechanisms would not result in a vector that is component-wise smaller than $\vec{b}(D)$. Instead, we extend the truncated Laplace mechanism to a multi-dimensional setting to compute $\bar{\vec{b}}(D)$.

As our main contribution, we prove that this approach is nearly optimal: we provide upper and lower bounds showing that the difference between the objective value of our algorithm's solution and the optimal solution is tight up to a factor of $O(\ln m)$ among all differentially private algorithms.
First, we present an upper bound on the utility of our algorithm. We prove that if $\vec{x}(D) \in \mathbb{R}^n$ is our algorithm's output and $\vec{x}^*$ is the optimal solution to the original optimization problem, then $g(\vec{x}(D))$ is close to $g\left(\vec{x}^*\right)$. Our bound depends on the \emph{sensitivity} $\Delta$ of the vector $\vec{b}(D)$, which equals the maximum $\ell_{1}$-norm between any two vectors $\vec{b}(D)$ and $\vec{b}(D')$ when $D$ and $D'$ are neighboring, in the sense that $D$ and $D'$ differ on at most one individual's data. Our bound also depends on the ``niceness'' of the matrix $\A$, which we quantify using the \emph{condition number} $\alpha(\A)$ of the linear system\footnote{Here, we use the simplified notation $\alpha(\A) = \inf_{p \geq 1} \left\{\alpha_{p,q}(\A)\sqrt[p]{m} \right\}$, where $\alpha_{p,q}(\A)$ is defined in Section~\ref{sec:multi_dim} and $\norm{\cdot}_q$ is the $\ell_q$-norm under which $g$ is $L$-Lipschitz.}
~\citep{Li93:Sharp, Mangasarian81:Condition}. We  summarize our upper bound below (see Theorem~\ref{thm:main} for the complete statement).
\begin{theorem}[Simplified upper bound]
	With probability 1, \begin{equation}g\left(\vec{x}^*\right) - g(\vec{x}(D))\leq \frac{2\cdot \Delta \cdot L \cdot \alpha(\A)}{\epsilon} \ln \left(\frac{m\left(e^{\epsilon} - 1\right)}{\delta} + 1\right).\label{eq:ub}\end{equation}
	\end{theorem}

We provide a lower bound that shows that Equation~\eqref{eq:ub} is tight up to a logarithmic factor.
	\begin{theorem}[Simplified lower bound]
	There is an infinite family of matrices $\A \in \R^{m \times m}$, a $1$-Lipschitz function $g: \R^m \to \R$, and a mapping from databases $D \subseteq \cX$ to vectors $\vec{b}(D) \in \R^m$ for any $\Delta > 0$ such that:
	\begin{enumerate}
		\item The sensitivity of $\vec{b}(D)$ equals $\Delta$, and
		\item For any $\epsilon > 0$ and $\delta \in (0, 1/2]$, if $\vec{\mu}$ is an $(\epsilon, \delta)$-differentially private mechanism such that $\A\vec{\mu}(D) \leq \vec{b}(D)$ with probability 1, then \[g(\vec{x}^*)  - \E[g(\vec{\mu}(D))] \geq \frac{\Delta \cdot \alpha(\A)}{4\epsilon} \ln \left(\frac{e^{\epsilon} - 1}{2\delta} + 1\right).\]
	\end{enumerate}
\end{theorem}
This lower bound matches  the upper bound from Equation~\eqref{eq:ub} up to a multiplicative factor of $O(\ln m)$.
See Theorem~\ref{thm:lb} for the complete statement.

\paragraph{Pure differential privacy.} A natural question is whether we can achieve pure $(\epsilon, 0)$-DP. In Appendix~\ref{app:characterization}, we prove that if $S^* := \bigcap_{D \subseteq \cX} \left\{\vec{x} : \A \vec{x} \leq \vec{b}(D) \right\}$---the intersection of the feasible regions across all databases $D$---is nonempty, then the optimal $(\epsilon, 0)$-differentially private mechanism disregards the database $D$ and outputs $\argmax_{\vec{x} \in S^*} g(\vec{x})$ with probability 1. If $S^* = \emptyset$, then no $(\epsilon, 0)$-differentially private mechanism exists. Therefore, any non-trivial private mechanism must allow for a failure probability $\delta > 0$. 

\paragraph{Experiments.} We empirically evaluate our algorithm in the contexts of financial portfolio optimization and internet advertising. Our experiments show that our algorithm can achieve nearly optimal performance while preserving privacy. We also compare our algorithm to a baseline $(\epsilon, 0)$-differentially private mechanism that is allowed to violate the problem's constraints. Our experiments demonstrate that for small values of the privacy parameter $\epsilon$, using the baseline algorithm yields a large number of violated constraints, while using our algorithm violates no constraints and incurs virtually no loss in revenue.

\subsection{Additional related research}

\paragraph{Truncated Laplace mechanism.}
Many papers have employed the truncated Laplace mechanism for various problems~\citep[e.g.,][]{Zhang12:Functional,Rinott18:Confidentiality,Bater18:Shrinkwrap,Croft19:Differential,Holohan20:Bounded,Geng20:Tight}. Our main contribution is not the use of this tool, but rather our proof that the truncated Laplace mechanism is the ``right'' tool to use for our optimization problem, with upper and lower bounds that match up to logarithmic factors.

Out of all papers employing the truncated Laplace mechanism, the one that is the most closely related to ours is by \citet{Geng20:Tight}, who study this mechanism in a one-dimensional setting. Given a query $q$ mapping from databases $D$ to $\R$, they study \emph{query-output independent noise-adding (QIN) algorithms}. Each such algorithm $\mu$ is defined by a distribution $P$ over $\R.$ It releases the query output $q(D)$ perturbed by additive random noise $X \sim P$, i.e., $\mu(D) = q(D) + X.$ 
They provide upper and lower bounds on the expected noise magnitude $|X|$ of any QIN algorithm, the upper bound equaling the expected noise magnitude of the truncated Laplace mechanism. They show that in the limit as the privacy parameters $\epsilon$ and $\delta$ converge to zero, these upper and lower bounds converge.

The Laplace mechanism is known to be a nearly optimal, general purpose $(\epsilon, 0)$-DP mechanism. While other task-specific mechanisms can surpass the utility of the Laplace mechanism \citep{Staircase}, they all induce distributions with exponentially decaying tails. The optimality of these mechanisms comes from the fact that the ratio between the mechanism's output distributions for any two neighboring databases is exactly $\exp(\epsilon)$. Adding less noise would fail to maintain that ratio everywhere, while adding more noise would distort the query output more than necessary.
\citet{Geng20:Tight} observe that in the case of $(\epsilon, \delta)$-DP mechanisms, adding large magnitude, low probability noise is wasteful, since the DP criteria can instead be satisfied using the $\delta$ ``budget'' rather than maintaining the $\exp(\epsilon)$ ratio everywhere.
To solve our private optimization problem, we shift and add noise to the constraints, and in our case adding large magnitude, low probability noise is not only wasteful but will cause the constraints to be violated.

Given their similar characterizations, it is not surprising that our mechanism is closely related to that of \citet{Geng20:Tight}---the mechanisms both draw noise from a truncated Laplace distribution.
The proof of our mechanism's optimality, however, is stronger than that of \citeauthor{Geng20:Tight}'s in several ways. First, it holds for any differentially private algorithm, not just for the limited class of QIN algorithms. Second, in the one-dimensional setting $(m=1)$---which is the setting that \citet{Geng20:Tight} analyze---our lower bound matches our algorithm's upper bound up to a constant factor of 8 for any $\epsilon > 0$ and $\delta \in (0, 1/2]$, not only in the limit as $\epsilon$ and $\delta$ converge to zero.

\paragraph{Private convex optimization.} There are multiple algorithmic approaches to differentially private convex optimization. Among others, these approaches include output and objective perturbation \citep{Chaudhuri11:Differentially},
the exponential mechanism \citep{Bassily14:ERM}, and private stochastic gradient descent \citep{BassilyFTT19}. The optimization problems tackled by these papers are either unconstrained, or the constraints are public information~\citep{Bassily14:ERM}. By contrast, the problems we show how to solve have private constraints. While Lagrange multipliers can transform a constrained problem into an unconstrained problem, we are not aware of a principled method for selecting Lagrange multipliers that would ensure constraint satisfaction. In fact, to \emph{privately} find the correct multiplier seems to be an equivalent problem to the one we are proposing.

To the best of our knowledge, only \citet{Hsu14:Privately} and~\citet{Cummings15:Privacy} have studied optimization problems with private constraints. They focus on linear programs where the constraint matrix $\A$, constraint vector $\vec{b}$, and linear objective function may depend on private data. These papers provide algorithms that are allowed to violate the constraints, but they guarantee that each constraint will not be violated by more than some amount, denoted $\alpha$, with high probability. Knowing this, an analyst could decrease each constraint by a factor of $\alpha$, and then be guaranteed that with high probability, the constraints will not be violated. Compared to that approach, our algorithm has several notable advantages. First, it is not \emph{a priori} clear what the loss in the objective value will be using their techniques, whereas we provide a simple approach with upper and lower bounds on the objective value loss that match up to logarithmic factors. Second, that approach only applies to linear programming, whereas we study the more general problem of linearly-constrained optimization. Finally, we guarantee that the constraints will not be violated with probability 1, whereas that approach would only provide a high probability guarantee. In Appendix~\ref{app:related}, we provide additional comparisons with \citeauthor{Hsu14:Privately}'s analysis (namely, the dichotomy between \emph{high-} and \emph{low-sensitivity} linear programs).

\paragraph{Differentially private combinatorial optimization.}
Several papers have studied differentially private combinatorial optimization~\citep{Gupta10:Differentially, Hsu16:Private}, which is a distinct problem from ours, since most combinatorial optimization problems cannot be formulated only using linear constraints. \citet{Hsu16:Private} study a private variant of a classic allocation problem: there are $n$ agents and $k$ goods, and the agents' values for all $2^k$ bundles of the goods are private. The goal is to allocate the goods among the agents in order to maximize social welfare, while maintaining differential privacy. This is similar but distinct from the transportation problem from Figure~\ref{fig:transportation}: if we were to follow the formulation from \citet{Hsu16:Private}, the transportation costs would be private, whereas in our setting, the transportation costs are public but the total demand of each hospital is private.

\section{Differential privacy definition}
To define differential privacy (DP), we first formally introduce the notion of a neighboring database: two databases $D, D' \subseteq \cX$ are \emph{neighboring}, denoted $D \sim D'$, if they differ on any one record $(|D\ \Delta\ D'| = 1)$. We use the notation $\vec{x}(D) \in \R^n$ to denote the random variable corresponding to the vector that our algorithm releases (non-trivial DP algorithms are, by necessity, randomized). Given privacy parameters $\epsilon \geq 0$ and $\delta \in [0,1]$, the algorithm satisfies \emph{$(\epsilon, \delta)$-differential privacy ($(\epsilon, \delta)$-DP)} if for any neighboring databases $D, D'$ and any subset $V \subseteq \R^n$, \[\Pr[\vec{x}(D) \in V] \leq e^{\epsilon} \Pr[\vec{x}(D') \in V] + \delta.\]

\section{Multi-dimensional optimization}\label{sec:multi_dim}
Our goal is to privately solve multi-dimensional optimization problems of the form \begin{equation}\max_{\vec{x} \in \R^n} \left\{g(\vec{x}) : \A\vec{x} \leq \vec{b}(D)\right\},\label{eq:original_opt}\end{equation} where $\vec{b}(D) = \left(b(D)_1, \dots, b(D)_m\right)$ is a vector in $\R^m$ and $g$ is an $L$-Lipschitz function according to an $\ell_q$-norm $||\cdot||_{q}$ for $q \geq 1$. 
        Preserving privacy while ensuring the constraints are always satisfied is impossible if the feasible regions change drastically across databases. For example, if $D$ and $D'$ are neighboring databases with disjoint feasible regions, there is no $(\epsilon, \delta)$-DP mechanism that always satisfies the constraints with $\delta < 1$ (see Lemma~\ref{lem:disjoint} in Appendix~\ref{app:multi_dim}).
        To circumvent this impossibility, we assume that the intersection of the feasible regions across databases is nonempty. This is satisfied, for example, if the origin is always feasible.
        \begin{assumption}\label{assumption:nonempty}
The set $S^* := \bigcap_{D \subseteq \cX} \left\{\vec{x} : \A \vec{x} \leq \vec{b}(D) \right\}$ is non-empty.
\end{assumption}

In our approach, we map each vector $\vec{b}(D)$ to a random variable $\bar{\vec{b}}(D) \in \R^m$ and release \begin{equation}\vec{x}(D) \in \argmax_{\vec{x} \in \R^n}\left\{g(\vec{x})  :
              \A \vec{x} \leq \bar{\vec{b}}(D)\right\}.\label{eq:private_LP}\end{equation}
To formally describe our approach, we use the notation $\Delta = \max_{D \sim D'} \norm{\vec{b}(D) - \vec{b}(D')}_1$ to denote the constraint vector's \emph{sensitivity}.
We define the $i^{th}$ component of $\bar{\vec{b}}(D)$ to be $\bar{b}(D)_i = \max\left\{b(D)_i - s + \eta_i, b_i^*\right\},$ where $s = \frac{\Delta}{\epsilon} \ln \left(\frac{m\left(e^{\epsilon} - 1\right)}{\delta} + 1\right)$, $\eta_i$ is drawn from the truncated Laplace distribution with support $[-s,s]$ and scale $\frac{\Delta}{\epsilon}$, and $b_i^* = \inf_D\left\{b(D)_i\right\}$. In Lemmas~\ref{lem:feasible} and \ref{lem:intersection} in Appendix~\ref{app:multi_dim}, we prove that $S^* = \left\{\vec{x} : \A \vec{x} \leq \left(b_1^*, \dots, b_m^*\right)\right\},$ which allows us to prove that Equation~\eqref{eq:private_LP} is feasible.

First, we prove that our algorithm satisfies differential privacy. We use the notation $\vec{\eta} = \left(\eta_1, \dots, \eta_m\right)$ to denote a random vector where each component is drawn i.i.d. from the truncated Laplace distribution with support $[-s,s]$ and scale $\frac{\Delta}{\epsilon}$. We also use the notation $\vec{b}(D) - s + \vec{\eta} = \left(b(D)_1 - s + \eta_1, \dots, b(D)_m- s + \eta_m\right)$. The proof of the following theorem is in Appendix~\ref{app:multi_dim}. 

\begin{restatable}{theorem}{private}\label{thm:private}
	The mapping $D \mapsto \vec{b}(D) - s + \vec{\eta}$ preserves $(\epsilon, \delta)$-differential privacy.
\end{restatable}
	
Since differential privacy is immune to post-processing~\citep{Dwork14:Algorithmic}, Theorem~\ref{thm:private} implies our algorithm is differentially private.

\begin{cor}
 The mapping $D \mapsto \vec{x}(D)$ is $(\epsilon,\delta)$-differentially private.
 \end{cor}
 
We next provide a bound on the quality of our algorithm, which measures the difference between the optimal solution $\max_{\vec{x} \in \R^n}\left\{g(\vec{x}) : \A \vec{x} \leq \vec{b}(D)\right\}$ and the solution our algorithm returns $g(\vec{x}(D))$.
 Our bound depends on the ``niceness'' of the matrix $\A$, as quantified by the linear system's \emph{condition number}~\citep{Li93:Sharp} $\alpha_{p, q}(\A)$. \citet{Li93:Sharp} proved that this value sharply characterizes the extent to which a change in the vector $\vec{b}$ causes a change in the feasible region, so it makes sense that it appears in our quality guarantees. Given a norm $||\cdot||_{p}$ on $\R^m$ where $p \geq 1$, we use the notation $||\cdot||_{p^*}$ to denote the dual norm where $\frac{1}{p} + \frac{1}{p^*} = 1$. The linear system's \emph{condition number} is defined as
\[\alpha_{p, q}(\A) = \sup_{\vec{u} \geq \vec{0}}\left\{\norm{\vec{u}}_{p^*} : \begin{array}{l}\norm{\A^T\vec{u}}_{q^*} = 1 
\text{ and the rows of } \A \text{ corresponding to}\\
\text{the nonzero entries of } \vec{u} \text{ are linearly independent}\end{array}\right\}.\]
 When $\A$ is nonsingular and $p = q = 2$, $\alpha_{p, q}(\A)$ is at most the inverse of the minimum singular value, $\sigma_{\min}(\A)^{-1}$. This value $\sigma_{\min}(\A)^{-1}$ is closely related to the matrix $\A$'s condition number (which is distinct from $\alpha_{p, q}(\A)$, the linear system's condition number), which roughly measures the rate at which the solution to $\A \vec{x} = \vec{b}$ changes with respect to a change in $\vec{b}$.
 
 We now prove our quality guarantee, which bounds the difference between the optimal solution to the original optimization problem (Equation~\eqref{eq:original_opt}) and that of the privately transformed problem (Equation~\eqref{eq:private_LP}).
 
 \begin{theorem}\label{thm:main}
 	Suppose Assumption~\ref{assumption:nonempty} holds and the function $g: \R^n \to \R$ is $L$-Lipschitz with respect to an $\ell_q$-norm $\norm{\cdot}_{q}$ on $\R^n$. With probability 1, \[\max_{\vec{x} \in \R^n}\left\{g(\vec{x}) : \A \vec{x} \leq \vec{b}(D)\right\} - g(\vec{x}(D))\leq\frac{2L\Delta}{\epsilon} \cdot \inf_{p \geq 1} \left\{\alpha_{p,q}(\A)\sqrt[p]{m} \right\} \cdot \ln \left(\frac{m \left(e^{\epsilon} - 1\right)}{\delta} + 1\right).\]
 \end{theorem}
 
 \begin{proof}
 	Let $\vec{b}$ be an arbitrary vector in the support of $\bar{\vec{b}}(D)$ and let $S =  \left\{\vec{x} : \A \vec{x} \leq \vec{b}\right\}$. Let $\vec{x}^*$ be an arbitrary point in $\argmax_{\vec{x} \in \R^n}\left\{g(\vec{x}) : \A \vec{x} \leq \vec{b}(D)\right\}$ and let $\bar{\vec{x}}$ be an arbitrary vector in $S$. We know that
 	\begin{align*}
 	&\max_{\vec{x} \in \R^n} \left\{g(\vec{x}) : \A \vec{x} \leq \vec{b}(D)\right\} - \max_{\vec{x} \in \R^n} \left\{g(\vec{x}) : \A \vec{x} \leq \vec{b}\right\}\\
 	=\text{ }&  g(\vec{x}^*) - \max_{\vec{x} \in \R^n} \left\{g(\vec{x}) : \A \vec{x} \leq \vec{b}\right\}\\
 	=\text{ }&g(\vec{x}^*) - g(\bar{\vec{x}}) + g(\bar{\vec{x}}) - \max_{\vec{x} \in \R^n} \left\{g(\vec{x}) : \A \vec{x} \leq \vec{b}\right\}.\end{align*} Since $\bar{\vec{x}} \in S = \{\vec{x} : \A \vec{x} \leq \vec{b}\}$, we know that $g(\bar{\vec{x}}) \leq \max_{\vec{x} \in \R^n} \left\{g(\vec{x}) : \A \vec{x} \leq \vec{b}\right\}.$ Therefore, \begin{equation}\max_{\vec{x} \in \R^n} \left\{g(\vec{x}) : \A \vec{x} \leq \vec{b}(D)\right\} - \max_{\vec{x} \in \R^n} \left\{g(\vec{x}) : \A \vec{x} \leq \vec{b}\right\}\leq  g(\vec{x}^*) - g(\bar{\vec{x}}) \leq L \cdot \norm{\vec{x}^* - \bar{\vec{x}}}_{q}. \label{eq:bound_M}
 	\end{equation}
 To simplify notation, let
 	$M
 	=g(\vec{x}^*) - \max \{g(\vec{x}) : \A \vec{x} \leq \vec{b}\}.$
 	Equation~\eqref{eq:bound_M} shows that for every $\bar{\vec{x}} \in S$, $\frac{M}{L} \leq \norm{\vec{x}^* - \bar{\vec{x}}}_{q}$. Meanwhile, from work by \citet{Li93:Sharp}, we know that for any $\ell_p$-norm $\norm{\cdot}_p$, \begin{equation}\inf_{\bar{\vec{x}} \in S}\norm{\vec{x}^* - \bar{\vec{x}}}_{q} \leq \alpha_{p, q}(\A) \cdot \norm{\vec{b}(D) - \vec{b}}_{p}.\label{eq:Li_opt}\end{equation} By definition of the infimum, this means that 
 	$M
 	\leq L  \cdot\alpha_{p, q}(\A) \cdot \norm{\vec{b}(D) - \vec{b}}_{p}.$
 	This inequality holds for any $\vec{b}$ in the support of $\bar{\vec{b}}(D)$ and with probability 1, \[\norm{\vec{b}(D) - \bar{\vec{b}}(D)}_p \leq \frac{2 \Delta \sqrt[p]{m}}{\epsilon}\ln \left(\frac{m\left(e^{\epsilon} - 1\right)}{\delta} + 1\right).\] Therefore, the theorem holds.
 	\end{proof}
 
  In the following examples, we instantiate Theorem~\ref{thm:main} in several specific settings.
 
 \begin{example}[Nonsingular constraint matrix]\label{ex:nonsingular}
 	When $\A$ is nonsingular, setting $\norm{\cdot}_p = \norm{\cdot}_q = \norm{\cdot}_2$ implies \[\max_{\vec{x} \in \R^n}\left\{g(\vec{x}) : \A \vec{x} \leq \vec{b}(D)\right\} - g(\vec{x}(D))\leq \frac{2 \cdot \Delta \cdot \sqrt{m} \cdot L}{\epsilon \cdot \sigma_{\min}(\A)} \ln \left(\frac{m\left(e^{\epsilon} - 1\right)}{\delta} + 1\right).\]
 \end{example}

 \begin{example}[Strongly stable linear inequalities]
 	We can obtain even stronger guarantees when the system of inequalities $\A \vec{x} < \vec{0}$ has a solution. In that case, the set $\{\vec{x} : \A \vec{x} \leq \vec{b}\}$ is non-empty for any vector $\vec{b}$~\citep{Mangasarian87:Lipschitz}, so we need not make Assumption~\ref{assumption:nonempty}. Moreover, when $\norm{\cdot }_{q}$ and $\norm{\cdot}_{p}$ both equal the $\ell_{\infty}$-norm and $\A \vec{x} < \vec{0}$ has a solution, we can replace $\alpha_{p, q}(\A)$ in Theorem~\ref{thm:main} with the following solution to a linear program: \[\bar{\alpha}(\A) = \max_{(\vec{u}, \vec{z}) \in \R^{m + n}}\left\{ \vec{1} \cdot \vec{u} : \vec{1} - \vec{z} \leq \vec{u}^\top \A \leq \vec{z},	\vec{u} \geq \vec{0},\text{ and }\vec{1} \cdot \vec{z} = 1\right\}.\] This is because in the proof of Theorem~\ref{thm:main}, we can replace Equation~\eqref{eq:Li_opt} with $\inf_{\bar{\vec{x}} \in S}\norm{\vec{x}^* - \bar{\vec{x}}}_{q} \leq \bar{\alpha}(\A) \cdot \norm{\vec{b}(D) - \vec{b}}_{p}$~\citep{Mangasarian87:Lipschitz}.
 \end{example}
 
 	We now present our main result. We prove that the quality guarantee from Theorem~\ref{thm:main} is tight up to a factor of $O(\log m)$.
 \begin{restatable}{theorem}{lb}\label{thm:lb}
 	Let $\A \in \R^{m \times m}$ be an arbitrary diagonal matrix with positive diagonal entries and let $g : \R^m \to \R$ be the function $g(\vec{x}) = \langle 1, \vec{x}\rangle$.
 	For any $\Delta > 0$, there exists
a mapping from databases $D \subseteq \cX$ to vectors $\vec{b}(D) \in \R^m$
  such that:
 	\begin{enumerate}
 		\item The sensitivity of $\vec{b}(D)$ equals $\Delta$, and
 		\item For any $\epsilon > 0$ and $\delta \in (0, 1/2]$, if $\vec{\mu}$ is an $(\epsilon, \delta)$-differentially private mechanism such that $\A\vec{\mu}(D) \leq \vec{b}(D)$ with probability 1, then \[\max\left\{g(\vec{x}) : \A\vec{x} \leq \vec{b}(D)\right\} - \E[g(\vec{\mu}(D))]\geq\frac{\Delta}{4\epsilon} \cdot \inf_{p \geq 1} \left\{\alpha_{p,1}(\A)\sqrt[p]{m} \right\} \cdot \ln \left(\frac{e^{\epsilon} - 1}{2\delta} + 1\right).\]
 	\end{enumerate}
 \end{restatable}
Since the objective function $g$ is 1-Lipschitz under the $\ell_1$-norm, this lower bound matches the upper bound from Theorem~\ref{thm:main} up to a factor of $O(\log m)$. The full proof of Theorem~\ref{thm:lb} is in Appendix~\ref{app:multi_dim}.
 
 \begin{proof}[Proof sketch of Theorem~\ref{thm:lb}]
 		For ease of notation, let $t = \frac{1}{\epsilon} \ln \left(\frac{e^{\epsilon} - 1}{2\delta} + 1\right).$ Notice that $\delta \leq \frac{1}{2}$ implies $t \geq 1$. 
 	For each vector $\vec{d} \in \Z^m$, let $D_{\vec{d}}$ be a database where for any $\vec{d}, \vec{d}' \in \Z^m$, if $\norm{\vec{d} - \vec{d}'}_1 \leq 1$, then $D_{\vec{d}}$ and $D_{\vec{d}'}$ are neighboring. Let $\vec{b}\left(D_{\vec{d}}\right) = \Delta \vec{d}$ and let $a_1, \dots, a_m > 0$ be the diagonal entries of $\A$. Since $\A\vec{\mu}\left(D_{\vec{d}}\right) \leq \vec{b}\left(D_{\vec{d}}\right)$ with probability 1, $\vec{\mu}\left(D_{\vec{d}}\right)$ must be coordinate-wise smaller than $\Delta \left(\frac{d_1}{a_1}, \dots, \frac{d_m}{a_m}\right)$.
 
 We begin by partitioning the support of $\vec{\mu}\left(D_{\vec{d}}\right)$ so that we can analyze $\E\left[g\left(\vec{\mu}\left(D_{\vec{d}}\right)\right)\right]$ using the law of total expectation. We organize this partition using axis-aligned rectangles.
Specifically, for each index $i \in [m]$, let $S_{i}^0$ be the set of vectors $\vec{x} \in \R^m$ whose $i^{th}$ components are smaller than $\frac{\Delta}{a_i}\left(d_i - \lfloor t \rfloor\right)$: \[S_{i}^0 = \left\{\vec{x} \in \R^m : x_i \leq \frac{\Delta}{a_i}\left(d_i - \lfloor t \rfloor\right)\right\}.\] Similarly, let \[S_{i}^1 = \left\{\vec{x} \in \R^m : \frac{\Delta}{a_i}\left(d_i - \lfloor t \rfloor\right) < x_i \leq \frac{\Delta d_i}{a_i}\right\}.\]
See Figure~\ref{fig:partition} for an illustration of these regions.
\begin{figure}
    \centering
    \includegraphics{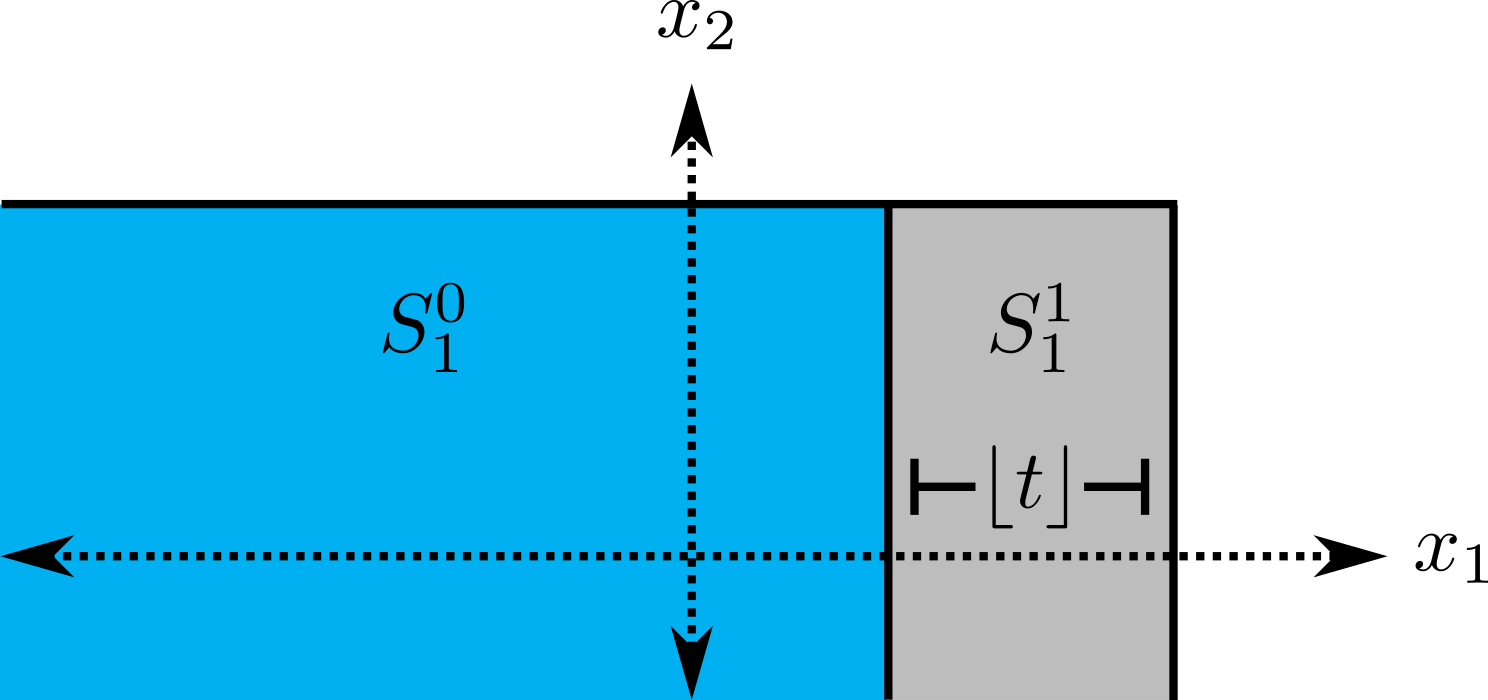}
    \caption{This figure illustrates the partition of $\R^2$ into $S_1^0$ (the left blue shaded region) and $S_1^1$ (the right grey shaded region). Assuming $\A$ is the identity matrix, the right vertical edge of $S_1^1$ lines up with $x_1 = d_1$ and the left vertical edge of $S_1^1$ lines up with $x_1 = d_1 - \lfloor t \rfloor$. The top horizontal edges of both $S_1^0$ and $S_1^1$ line up with $x_2 = d_2$.}
    \label{fig:partition}
\end{figure}
For any vector $\vec{I} \in \{0,1\}^m$, let $S_{\vec{I}} = \cap_{i = 1}^m S_i^{I_i}$. The sets $S_{\vec{I}}$ partition the support of $\vec{\mu}\left(D_{\vec{d}}\right)$ into rectangles.
Therefore, by the law of total expectation, \begin{equation}\E\left[g\left(\vec{\mu}\left(D_{\vec{d}}\right)\right)\right] = \sum_{{\vec{I}} \in \{0,1\}^m}\E\left[g\left(\vec{\mu}\left(D_{\vec{d}}\right)\right) \mid \vec{\mu}\left(D_{\vec{d}}\right) \in S_{{\vec{I}}}\right]\mathbb{P}\left[\vec{\mu}\left(D_{\vec{d}}\right) \in S_{{\vec{I}}}\right].\label{eq:law}\end{equation}
 
When we condition on the vector $\vec{\mu}\left(D_{\vec{d}}\right)$ being contained in a rectangle $S_{{\vec{I}}}$, our analysis of the expected value of $g\left(\vec{\mu}\left(D_{\vec{d}}\right)\right)$ is simplified.
 Suppose that $\vec{\mu}\left(D_{\vec{d}}\right) \in S_{{\vec{I}}}$ for some $\vec{I} \in \{0,1\}^m$. If $I_i = 0$, then we know that $\mu\left(D_{\vec{d}}\right)_i \leq \frac{\Delta}{a_i}\left(d_i - \lfloor t \rfloor\right)$. Meanwhile, if $I_i = 1$, then $\mu\left(D_{\vec{d}}\right)_i \leq \frac{\Delta d_i}{a_i}$ since $\A \vec{\mu}\left(D_{\vec{d}}\right) \leq \vec{b}\left(D_{\vec{d}}\right)$ with probability 1. Since $g(\vec{x}) = \langle \vec{1}, \vec{x} \rangle$, we have that for each $\vec{I} \in \{0,1\}^m$,
\[ \E\left[g\left(\vec{\mu}\left(D_{\vec{d}}\right)\right) \mid \vec{\mu}\left(D_{\vec{d}}\right) \in S_{{\vec{I}}}\right]\leq \sum_{i = 1}^m \frac{\Delta\left(d_i - \lfloor t \rfloor\right)}{a_i}\textbf{1}_{\{I_i = 0\}} + \frac{\Delta d_i}{a_i}\textbf{1}_{\{I_i = 1\}}=\sum_{i = 1}^m \frac{\Delta d_i}{a_i}-  \frac{\Delta\lfloor t \rfloor}{a_i}  \textbf{1}_{\{I_i = 0\}}.\]
 Combining this inequality with Equation~\eqref{eq:law} and rearranging terms, we are able to prove that \begin{equation}\E\left[g\left(\vec{\mu}\left(D_{\vec{d}}\right)\right)\right]\leq \Delta\sum_{i = 1}^m \frac{d_i}{a_i} - \Delta \lfloor t \rfloor \sum_{i = 1}^m\frac{1}{a_i}\Pr\left[\vec{\mu}\left(D_{\vec{d}}\right) \in S_i^0\right]\label{eq:bound_exp}\end{equation} (see the full proof in Appendix~\ref{app:multi_dim} for details).

We use the definition of differential privacy to show that for all $i \in [m]$, $\Pr\left[\vec{\mu}\left(D_{\vec{d}}\right) \in S_i^0\right] > \frac{1}{2}$, which allows us to simplify Equation~\eqref{eq:bound_exp}. Intuitively this holds since $\vec{\mu}\left(D_{\vec{d}}\right)$ cannot have too much probability mass in each set $S_i^1$, as there are neighboring databases that have zero probability mass in subsets of this region. More precisely, we show that $\Pr\left[\vec{\mu}\left(D_{\vec{d}}\right) \in S_i^0\right] > \delta \sum_{j = 0}^{\lfloor t \rfloor - 1} e^{\epsilon j}$. Our choice of $t$ then implies that $\Pr\left[\vec{\mu}\left(D_{\vec{d}}\right) \in S_i^0\right] > \frac{1}{2}$.

 	This inequality, Equation~\eqref{eq:bound_exp}, and the fact that $t \geq 1$ together imply that \[
 	\E\left[g\left(\vec{\mu}\left(D_{\vec{d}}\right)\right)\right] < \Delta\sum_{i = 1}^m \frac{d_i}{a_i} - \frac{\Delta t}{4}\sum_{i = 1}^m\frac{1}{a_i}.\] Since $\max\left\{g(\vec{x}) : \A \vec{x}\leq\vec{b}\left(D_{\vec{d}}\right) \right\} = \Delta \sum_{i= 1}^m \frac{d_i}{a_i},$ we have that
 	\[\max\left\{g(\vec{x}) : \A \vec{x}\leq\vec{b}\left(D_{\vec{d}}\right) \right\}  - \E\left[g\left(\vec{\vec{\mu}}\left(D_{\vec{d}}\right)\right)\right]\geq \frac{\Delta }{4\epsilon}\left(\sum_{i = 1}^m \frac{1}{a_i}\right) \ln \left(\frac{e^{\epsilon} - 1}{2\delta} + 1\right).\]
 	
Finally, we prove that $\inf_{p \geq 1}\alpha_{p,1}(\A) \sqrt[p]{m} \leq \alpha_{\infty,1}(\A)  = \sum_{i = 1}^m \frac{1}{a_i}$, which implies that the theorem statement holds.
 	Since $\A$ is diagonal,  \[\alpha_{\infty,1}(\A) = \sup_{\vec{u} \geq \vec{0}}\left\{\norm{\vec{u}}_{1} : u_{i}a_{i} \leq 1, \forall i \in [m]\right\} = \sum_{i = 1}^m \frac{1}{a_i}.\]  Moreover, since $\alpha_{\infty,1}(\A) \in \left\{\alpha_{p,1}(\A) \sqrt[p]{m} : p \geq 1\right\}$, we have that $\inf_{p \geq 1}\alpha_{p,1}(\A) \sqrt[p]{m} \leq \alpha_{\infty,1}(\A).$ Therefore,  the theorem statement holds.
 \end{proof}

This theorem demonstrates that our algorithm's loss (Theorem~\ref{thm:main}) is tight up to a factor of $O(\log m)$ among all  differentially private mechanisms.

\section{Experiments}\label{sec:experiments}
In this section, we present empirical evaluations of our algorithm in several settings: financial portfolio optimization and internet advertising.

\subsection{Portfolio optimization}\label{sec:portfolio} Suppose a set of individuals pool their money to invest in a set of $n$ assets over a period of time. The amount contributed by each individual is private, except to the trusted investment manager. Let $b(D)$ be the total amount of money the investors (represented by a database $D$) contribute.  We let $x_i$ denote the amount of asset $i$ held throughout the period, with $x_i$ in dollars, at the price at the beginning of the period. We adopt the classic \citet{Markowitz52:Portfolio} portfolio optimization model. The return of each asset is represented by the random vector $\vec{p} \in \R^n$, which has known mean $\bar{\vec{p}}$ and covariance $\Sigma$. Therefore with portfolio $\vec{x} \in \R^n$, the return $r$ is a (scalar) random variable with mean $\bar{\vec{p}} \cdot \vec{x}$ and variance $\vec{x}^{\top} \Sigma \vec{x}$. The choice of
a portfolio $\vec{x}$ involves a trade-off between the return's mean and variance. Given a minimum return $r_{min}$, the goal is to solve the following quadratic program while keeping the budget $b(D)$ private: \begin{equation}\begin{array}{ll}
\text{minimize} & \vec{x}^{\top} \Sigma \vec{x}\\
\text{such that} & \bar{\vec{p}} \cdot \vec{x} \geq r_{min}\\
& \vec{x} \cdot \vec{1} \leq b(D)\\
& \vec{x} \geq \vec{0}.
\end{array}\label{eq:portfolio}\end{equation}

We run experiments using real-world data from stocks included in the Dow Jones Industrial Average, compiled by~\citet{Bruni16:Real}. They collected weekly linear returns for 28 stocks over the course of 1363 weeks. The mean vector $\bar{\vec{p}} \in \R^{28}$ is the average of these weekly returns and the covariance matrix $\Sigma \in \R^{28 \times 28}$ is the covariance of the weekly returns.

\begin{figure}
	\centering
	\includegraphics[scale=.7]{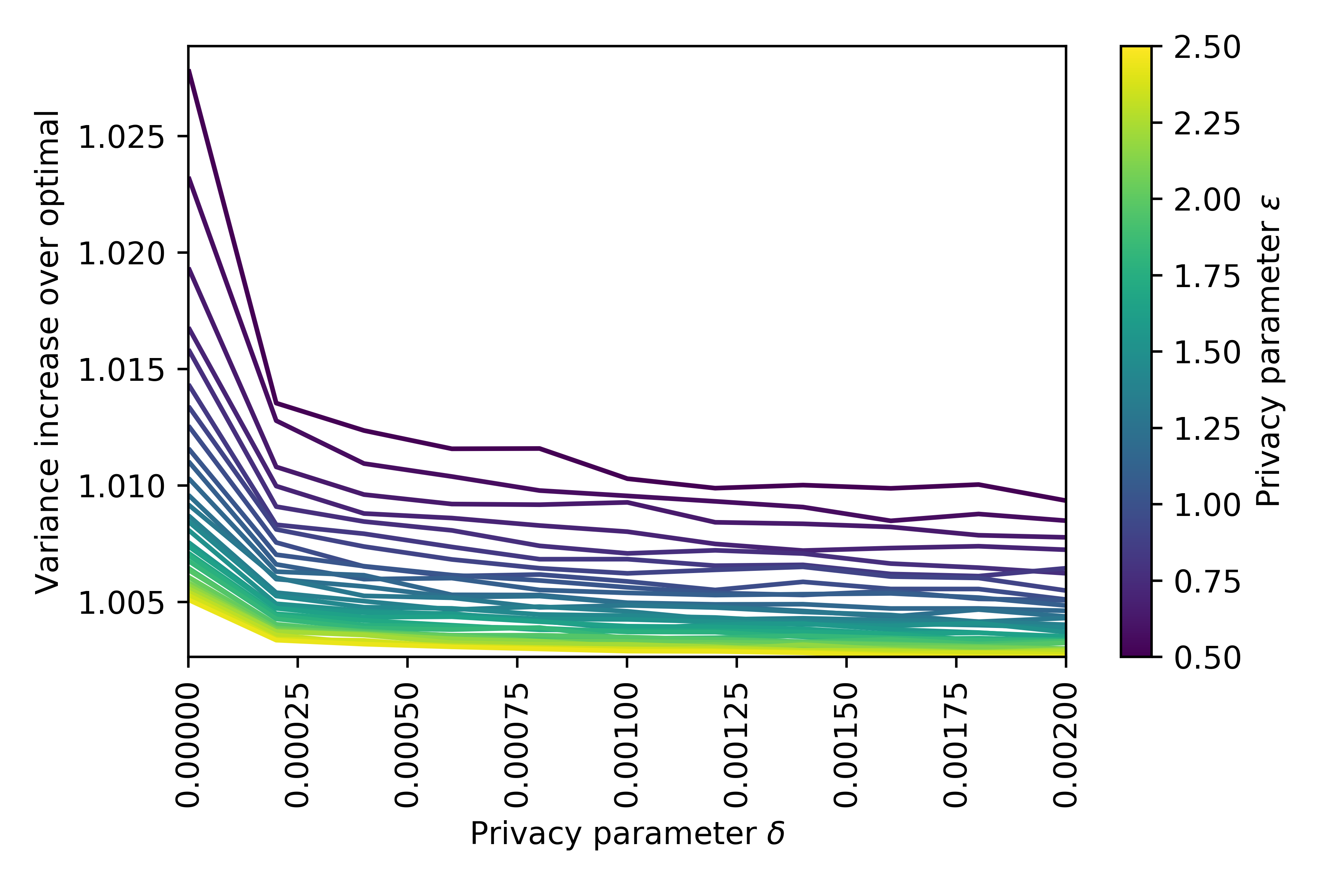}\centering
	\caption{Quality in the portfolio optimization application.  The plot shows the multiplicative increase in the objective function value of our algorithm's solution---for various choices of $\epsilon$ and $\delta$---over the objective function value of the optimal solution to the original optimization problem (Equation~\eqref{eq:portfolio}). Darker shading corresponds to lower values of $\epsilon$ and hence stronger privacy. See Section~\ref{sec:portfolio} for details.}
	\label{fig:quality}
\end{figure}

In Figure~\ref{fig:quality}, we analyze the quality of our algorithm. First, we set the number of individuals $n$ to be $1000$. Then, we define each element of the database (money given by individuals to an investor) as a draw from the uniform distribution between 0 and 1, so $b(D)$ equals the sum of these $n$ random variables. The sensitivity of $b(D)$ is therefore $\Delta = 1$. We set the minimum return $r_{min}$ to be $2.5$.
We calculate the objective value $v^* \in \R$ of the optimal solution to Equation~\eqref{eq:portfolio}. Then, for $\delta \in \left[\frac{1}{n^2}, 0.002\right]$ and $\epsilon \in [0.5, 2.5]$, we run our algorithm 50 times and calculate the average objective value $\hat{v}_{\epsilon, \delta} \in \R$ of the optimal solutions.

In Figure~\ref{fig:quality}, we plot $\frac{\hat{v}_{\epsilon, \delta}}{v^*}$. We see that even strict values for the privacy parameters do not lead to a significant degradation in the value of the objective function. For example, setting $\epsilon = 0.5$ and $\delta = 2.5 \cdot 10^{-4}$ increases the value of the objective function by about $1\%$.

In Appendix~\ref{app:experiments}, we perform the same experiment with the number $n$ of investors in $\{500, 1000, 1500\}$ and the minimum return $r_{min}$ in the interval $[1,5]$. We obtain plots that are similar to Figure~\ref{fig:quality}.
As we describe in Appendix~\ref{app:experiments}, we find that there is a sweet spot for the parameter choices $n$ and $r_{min}$. If $r_{min}$ is too small, the budget constraint is non-binding with or without privacy, so the variance increase over optimal is always 1. Meanwhile, if $r_{min}$ is too large, then the original quadratic program (Equation~\eqref{eq:portfolio}) is infeasible.

\subsection{Internet advertising}
\label{sec:advertising}

Many internet publishers hire companies called \emph{supply-side platforms (SSPs)} to manage their advertising inventory. A publisher using an SSP partitions its website's pages into $M$ groups, and informs the SSP of the number $n_j$ of impressions (i.e., visiting users) available in each group $j$. For example, an online newspaper might have a sports section, a fashion section, and so on. The SSP relays the list of inventory groups to $N$ potential advertisers, and each advertiser $i$ responds with the monetary amount $c_{ij} \ge 0$ they are willing to pay per impression from each group $j$, and also their budget $b(D)_i \ge 0$ for overall spending on the publisher's website, where $D$ represents advertisers' confidential business information, such as their financial health and strategic priorities. The SSP then allocates $x_{ij}$ impressions from each group $j$ to each advertiser $i$ so as to maximize the publisher's revenue while respecting both the impression supply constraints and advertiser budget constraints:
\begin{equation}\arraycolsep=1.4pt\def\arraystretch{1.5}
\begin{array}{lll}
\text{maximize} & \sum_{ij} c_{ij} x_{ij}\\
\text{such that} & \sum_{i = 1}^N x_{ij} \le n_j & \text{ for } j \in [M] \\
& \sum_{j = 1}^M c_{ij} x_{ij} \le b(D)_i & \text{ for } i \in [N]\\
& x_{ij} \ge 0.
\end{array}
\label{eq:advertising}
\end{equation}
This linear program is similar to the transportation problem from Figure~\ref{fig:transportation}.

Existing algorithms for private optimization are not guaranteed to output a solution that satisfies all the problem constraints, so we explore how often those algorithms violate the constraints when applied to the advertising problem in Equation~\eqref{eq:advertising}. The algorithm most closely related to ours is by \citet{Hsu14:Privately}, but our settings do not quite match: they require that the optimal solution have constant norm across all possible private database. If this is not the case (and it is not for our advertising problem), \citet{Hsu14:Privately} recommend normalizing the problem parameters by the norm of the optimal non-private solution (which itself is a sensitive value). However, this will necessarily impact the problem's sensitivity parameter $\Delta$, and \citet{Hsu14:Privately} do not provide guidance on how to quantify this impact, though knowing this sensitivity is crucial for running the algorithm.

Therefore, we compare our algorithm with an alternative baseline.
We run experiments that use two algorithms to transform each advertiser's budget $b(D)_i$ in Equation~\eqref{eq:advertising} to a private budget $\bar{b}(D)_i$. Both algorithms set $\bar{b}(D)_i = \max\left\{b(D)_i - s + \eta, 0\right\}$, where $s = \frac{\Delta}{\epsilon} \ln \left(\frac{N (e^{\epsilon} - 1)}{\delta}  + 1\right)$ for privacy parameters $\epsilon, \delta$ and sensitivity parameter $\Delta$, and $\eta$ is a random variable. The first algorithm follows our method described in Section \ref{sec:multi_dim} and draws $\eta$ from the truncated Laplace distribution with support $[-s,s]$ and scale $\frac{\Delta}{\epsilon}$. The baseline algorithm instead draws $\eta$ from the Laplace distribution with scale $\frac{\Delta}{\epsilon}$, and thus is $(\epsilon,0)$-differentially private. Both algorithms use noise distributions with roughly the same shape, but only our algorithm is guaranteed to satisfy the original constraints.

Our experiments consist of simulations with parameters chosen to resemble real data from an actual SSP. The publisher has $M = 200$ inventory groups, and there are $N = 10$ advertisers who wish to purchase inventory on the publisher's website. The amount $c_{ij}$ each advertiser $i$ is willing to pay per impression from each group $j$ is \$0 with probability 0.2, and drawn uniformly from $[\$0, \$1]$ with probability 0.8. The number of impressions $n_j$ per group $j$ is $10^7$, and each advertiser's budget $b(D)_i$ is drawn uniformly from $[\$10^7 - \Delta/2, \$10^7 + \Delta/2]$, where $\Delta = \$10^2$ is also the sensitivity of the budgets with respect to the private database $D$. The results for various values of the privacy parameter $\epsilon$ (with the privacy parameter $\delta$ fixed at $10^{-4}$) are shown in Figure~\ref{fig:advertising}, where every data point on the plot is the average of $400$ simulations.
\begin{figure}
     \centering
	\includegraphics[width=.5\linewidth]{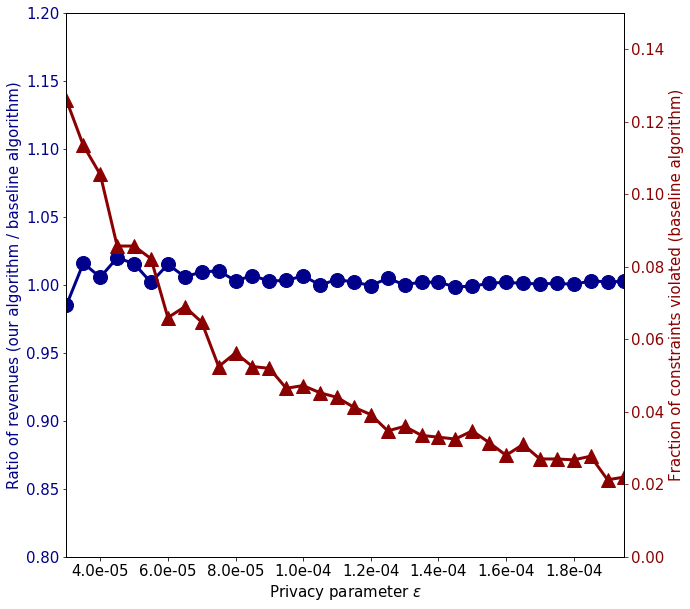}\centering
	\caption{Quality in the advertising application. Ratio of the revenue of our algorithm's solution and that of the baseline algorithm's solution (circle markers, left vertical axis), and fraction of constraints in the original optimization problem (Equation~\eqref{eq:advertising}) violated by the baseline algorithm (triangle markers, right vertical axis). See Section~\ref{sec:advertising} for details.}\label{fig:advertising}
\end{figure}

Figure~\ref{fig:advertising} shows that for small values of $\epsilon$, using the baseline algorithm yields a large number of violated constraints, while using our algorithm violates no constraints and incurs virtually no loss in revenue.

\section{Conclusions}
We presented a differentially private method for solving linearly-constrained optimization problems, where the right-hand side of the constraints $\A\vec{x} \le \vec{b}$ depends on private data, and where the constraints must always be satisfied. We showed that our algorithm is nearly optimal: its loss is tight up to a factor of $O(\log m)$ among all DP algorithms. Empirically, we used real and synthetic datasets to show that our algorithm returns nearly optimal solutions in realistic settings. A natural direction for future research would be to allow the matrix $\A$ to also depend on private data.

\subsection*{Acknowledgments}
This work was supported in part by a fellowship from Carnegie Mellon University’s Center for Machine Learning and Health and an IBM PhD Fellowship to E.V.

\bibliography{VitercikLibrary}
\bibliographystyle{plainnat}

\appendix

\section{Additional details about related research}\label{app:related}
In \citeauthor{Hsu14:Privately}'s paper on private linear programming~\citep{Hsu14:Privately}, the authors study two different categories of linear program (LPs) when the constraint vector $\vec{b}(D)$ depends on private data:
\begin{enumerate}
    \item \emph{High-sensitivity LPs}: For any two neighboring databases $D$ and $D'$, there is exactly one component $i \in [m]$ where $\vec{b}(D)_i \not= \vec{b}(D')_i$, and for every other component $j \not= i$, $\vec{b}(D)_j = \vec{b}(D')_j$.
    \item \emph{Low-sensitivity LPs}: For any two neighboring databases $D$ and $D'$ of size $N$, $\norm{\vec{b}(D) - \vec{b}(D')}_{\infty} \leq \frac{1}{N}$.
\end{enumerate}

They prove that in general, high-sensitivity LPs cannot be solved privately. Specifically, for any database $D \in \{0,1\}^n$ of size $n$, they define the following high-sensitivity LP:
\[\begin{array}{lll}
\text{find} & \vec{x}\\
\text{such that} & x_i = D_i & \text{for all } i \in [n].
\end{array}\]
They prove that for any $(\epsilon, \delta)$-differentially private mechanism with output $\vec{x}(D)$, there will be at least one component $i \in [n]$ such that $\left|\vec{x}(D)_i - D_i\right| \geq \frac{1}{2}$ (otherwise, the mechanism would be able to reconstruct $D$ exactly). This fact does not contradict our upper bound from Theorem~\ref{thm:main} since this worst-case problem does not satisfy Assumption~\ref{assumption:nonempty}.

For low-sensitivity LPs, they show that a private version of the multiplicative weights algorithm returns a solution that is close to the LP's optimal solution. Their solution is allowed to violate the LP's constraints, so their algorithm does not apply in our setting.

\section{Omitted proofs from Section~\ref{sec:multi_dim} about multi-dimensional optimization}\label{app:multi_dim}
\private*

\begin{proof}
	Let $D$ and $D'$ be two neighboring databases. We write the density function of $\vec{b}(D) - s + \vec{\eta}$ as $f_D(\vec{u}) \propto \prod_{i = 1}^m\exp\left(-\frac{\epsilon|u_i + s - b(D)_i|}{\Delta}\right)$ when $\vec{u} \in [\vec{b}(D) - 2s, \vec{b}(D)]$ and $f_D(\vec{u}) = 0$ when $\vec{u} \not\in [\vec{b}(D) - 2s, \vec{b}(D)]$. This proof relies on the following two claims. The first claim shows that in the intersection of the supports $[\vec{b}(D) - 2s, \vec{b}(D)] \cap [\vec{b}(D') - 2s, \vec{b}(D')],$ the density functions $f_D$ and $f_{D'}$ are close.
	
	\begin{claim}\label{claim:eps}
		Let $\vec{u}$ be a vector in the intersection of the supports $[\vec{b}(D) - 2s, \vec{b}(D)] \cap [\vec{b}(D') - 2s, \vec{b}(D')]$. Then $f_{D}(\vec{u}) \leq e^{\epsilon}f_{D'}(\vec{u})$.
		\end{claim}
		
			\begin{proof}[Proof of Claim~\ref{claim:eps}]
 Since $\vec{u}$ is a vector in the intersection of the support $[\vec{b}(D) - 2s, \vec{b}(D)] \cap [\vec{b}(D') - 2s, \vec{b}(D')]$, \begin{align*}\frac{f_D(\vec{u})}{f_{D'}(\vec{u})} &= \prod_{i = 1}^m\frac{\exp\left(-\epsilon|u_i + s - b(D)_i|/\Delta\right)}{\exp\left(-\epsilon|u_i + s - b(D')_i|/\Delta\right)}\\
		&= \prod_{i = 1}^m\exp\left(\frac{\epsilon\left(|u_i + s - b(D')_i| - |u_i + s - b(D)_i|\right)}{\Delta}\right)\\
		&\leq \prod_{i = 1}^m\exp\left(\frac{\epsilon|b(D)_i - b(D')_i|}{\Delta}\right)\\
		&=\exp\left(\frac{\epsilon\sum_{i = 1}^m|b(D)_i - b(D')_i|}{\Delta}\right)\\
			&\leq\exp\left(\frac{\epsilon\Delta}{\Delta}\right)\\
			&= e^{\epsilon},
		\end{align*}
	as claimed.
\end{proof}
		The second claim shows that the total density of $\vec{b}(D) - s + \vec{\eta}$ on vectors not contained in the support of $\vec{b}(D') - s + \vec{\eta}$ is at most $\delta$.

	\begin{claim}\label{claim:delta}
Let $V = [\vec{b}(D) - 2s, \vec{b}(D)] \setminus [\vec{b}(D') - 2s, \vec{b}(D')]$ be the set of vectors in the support of $\vec{b}(D) - s + \vec{\eta}$  but not in the support of $\vec{b}(D') - s + \vec{\eta}$. Then $\Pr[\vec{b}(D) - s + \vec{\eta} \in V] \leq \delta$.
		\end{claim}
	
	\begin{proof}[Proof of Claim~\ref{claim:delta}]
Suppose $\vec{b}(D) - s + \vec{\eta} \in V$. Then for some $i \in [m]$, either $b(D)_i - s + \eta_i < b(D')_i - 2s$ or $b(D)_i - s + \eta_i > b(D')_i$. This implies that either $\eta_i < -s + \Delta$ or $\eta_i > s - \Delta$. The density function of the truncated Laplace distribution with support $[-s,s]$ and scale $\frac{\Delta}{\epsilon}$ is  \[f(\eta) = \begin{cases}\frac{1}{Z}\exp\left(-\frac{|\eta|\epsilon}{\Delta}\right) &\text{if } \eta \in [-s, s]\\
	0 &\text{otherwise,}
	\end{cases}\] where $Z = \frac{2 \Delta \left(1 - e^{-\epsilon s/\Delta}\right)}{\epsilon}$ is a normalizing constant. Therefore, the probability that for some $i \in [m]$, either $\eta_i < -s + \Delta$ or $\eta_i > s - \Delta$ is \begin{align*}
		m\left(\int_{-s}^{-s + \Delta} f(\eta) \, d\eta + \int_{s - \Delta}^{s} f(\eta) \, d\eta\right)
		&= \frac{m}{Z}\left(\int_{-s}^{-s + \Delta} \exp\left(-\frac{|\eta|\epsilon}{\Delta}\right)\, d\eta + \int_{s - \Delta}^{s}\exp\left(-\frac{|\eta|\epsilon}{\Delta}\right)\, d\eta\right)\\
			&= \frac{2m\Delta (e^{\epsilon} - 1)e^{-s\epsilon/\Delta}}{Z\epsilon}\\
			&= \frac{m\left(e^{\epsilon} - 1\right)e^{-s \epsilon/\Delta}}{1 - e^{-\epsilon s/\Delta}}\\
			&=\frac{m\left(e^{\epsilon} - 1\right)}{e^{s \epsilon/\Delta} - 1}\\
			&=\delta,
		\end{align*} where the final equality follows from the fact that $s = \frac{\Delta}{\epsilon} \ln \left(\frac{m\left(e^{\epsilon} - 1\right)}{\delta} + 1\right).$ In turn, this implies that $\Pr[\vec{b}(D) - s + \vec{\eta} \in V] \leq \delta$.
	\end{proof}

These two claims imply that the mapping $D \mapsto \vec{b}(D) - s + \vec{\eta}$ preserves $(\epsilon, \delta)$-differential privacy. To see why, let $W \subseteq  [\vec{b}(D) - 2s, \vec{b}(D)]$ be an arbitrary set of vectors in the support of $\vec{b}(D) - s + \vec{\eta}$. Let $W_0 = W \cap [\vec{b}(D') - 2s, \vec{b}(D')]$ be the set of vectors in $W$ that are also in the support of $\vec{b}(D') - s + \vec{\eta}$ and let $W_1 = W \setminus [\vec{b}(D') - 2s, \vec{b}(D')]$ be the remaining set of vectors in $W$. As in Claim~\ref{claim:delta}, let $V = [\vec{b}(D) - 2s, \vec{b}(D)] \setminus [\vec{b}(D') - 2s, \vec{b}(D')]$ be the set of vectors in the support of $\vec{b}(D) - s + \vec{\eta}$  but not in the support of $\vec{b}(D') - s + \vec{\eta}$. Clearly, $W_1 \subseteq V$. Therefore, \begin{align*}\Pr[\vec{b}(D) - s + \vec{\eta} \in W] 	&=\Pr[\vec{b}(D) - s + \vec{\eta} \in W_0] + \Pr[\vec{b}(D) - s + \vec{\eta} \in W_1] \\
&	\leq \Pr[\vec{b}(D) - s + \vec{\eta} \in W_0] + \Pr[\vec{b}(D) - s + \vec{\eta} \in V] \\
		&=\int_{W_0}f_D(\vec{u})  \, d \vec{u}+ \int_{V}f_D(\vec{u})  \, d \vec{u}\\
				&\leq\int_{W_0}e^{\epsilon}f_{D'}(\vec{u})  \, d \vec{u}+ \int_{V}f_D(\vec{u})  \, d \vec{u} &\text{(Claim~\ref{claim:eps})}\\
&					\leq\int_{W_0}e^{\epsilon}f_{D'}(\vec{u})  \, d \vec{u}+ \delta &\text{(Claim~\ref{claim:delta})}\\
					&\leq e^{\epsilon}\Pr[\vec{b}(D') - s + \vec{\eta} \in W]+ \delta ,
		\end{align*}
	so differential privacy is preserved.
	\end{proof}

\begin{lemma}\label{lem:disjoint}
    Suppose $D$ and $D'$ are two neighboring databases with disjoint feasible regions: $\left\{\vec{x} : \A \vec{x} \leq \vec{b}(D)\right\} \cap \left\{\vec{x} : \A \vec{x} \leq \vec{b}(D')\right\} = \emptyset$. There is no $(\epsilon, \delta)$-DP mechanism with $\delta < 1$ that satisfies the constraints with probability 1.
\end{lemma}

\begin{proof}
For the sake of a contradiction, suppose $\mu : 2^{\cX} \to \R^n$ is an $(\epsilon, \delta)$-DP mechanism with $\delta < 1$  that satisfies the constraints with probability 1. Let $V = \left\{\vec{x} : \A \vec{x} \leq \vec{b}(D)\right\}$. Since $V \cap \left\{\vec{x} : \A \vec{x} \leq \vec{b}(D')\right\} = \emptyset$, it must be that $\Pr[\mu(D') \in V] = 0$. This means that $1 = \Pr[\mu(D) \in V] \leq e^{\epsilon} \Pr[\mu(D') \in V] + \delta = \delta$, which is a contradiction. Therefore, the lemma statement holds.
\end{proof}

\begin{lemma}\label{lem:feasible}
With probability 1, the optimization problem in Equation~\eqref{eq:private_LP} is feasible.
\end{lemma}
\begin{proof}
By definition, the constraint vector $\bar{\vec{b}}$ is component-wise greater than the vector $\vec{b}^* = \left(b_1^*, \dots, b_m^*\right)$, where $b_i^* = \inf_{D \subseteq \cX} b(D)_i.$ Therefore, $\left\{\vec{x} : \A \vec{x} \leq \bar{\vec{b}}(D)\right\} \supseteq \left\{\vec{x} : \A \vec{x} \leq \vec{b}^*\right\}$. By Lemma \ref{lem:intersection}, we know that $\left\{\vec{x} : \A \vec{x} \leq \vec{b}^*\right\} = \bigcap_{D \subseteq \cX} \left\{\vec{x} : \A \vec{x} \leq \vec{b}(D) \right\}$ and by Assumption~\ref{assumption:nonempty}, we know that $\bigcap_{D \subseteq \cX} \left\{\vec{x} : \A \vec{x} \leq \vec{b}(D) \right\}$ is nonempty. Therefore, the feasible set of the linear program in Equation~\eqref{eq:private_LP}, $\left\{\vec{x} : \A \vec{x} \leq \bar{\vec{b}}(D)\right\}$, is nonempty.
\end{proof}

We now prove Lemma~\ref{lem:intersection}, which we used in the proof of Lemma \ref{lem:feasible}. Lemma~\ref{lem:intersection} guarantees that the (nonempty) intersection of the feasible regions across all databases is equal to the set of all $\vec{x}$ such that $\A \vec{x} \leq \vec{b}^*.$

\begin{lemma}\label{lem:intersection}
    The set $\bigcap_{D \subseteq \cX} \left\{\vec{x} : \A \vec{x} \leq \vec{b}(D) \right\}$ is equal to the set $\left\{\vec{x} : \A \vec{x} \leq \vec{b}^*\right\}$.
    \end{lemma}
    \begin{proof}
        Suppose that $\vec{x} \in \bigcap_{D \subseteq \cX} \left\{\vec{x} : \A \vec{x} \leq \vec{b}(D) \right\}$. We claim that $\A \vec{x} \leq \vec{b}^*.$ To see why, let $\vec{a}_i$ be the $i^{th}$ row of the matrix $\A$. We know that for all datasets $D \subseteq \cX$, $\vec{a}_i \cdot \vec{x} \leq b(D)_i$. By definition of the infimum, this means that $\vec{a}_i \cdot \vec{x} \leq \inf_{D \subseteq \cX}b(D)_i = b_i^*$. Therefore, $\A \vec{x} \leq \vec{b}^*.$
        
        Next, suppose $\A \vec{x} \leq \vec{b}^*.$ Then $\A \vec{x} \leq \vec{b}(D)$ for every database $D$, which means that $\vec{x} \in \bigcap_{D \subseteq \cX} \left\{\vec{x} : \A \vec{x} \leq \vec{b}(D) \right\}$. We conclude that $\bigcap_{D \subseteq \cX} \left\{\vec{x} : \A \vec{x} \leq \vec{b}(D) \right\} = \left\{\vec{x} : \A \vec{x} \leq \vec{b}^*\right\}$.
    \end{proof}

 	\lb*
 	 \begin{proof}
 		For ease of notation, let $t = \frac{1}{\epsilon} \ln \left(\frac{e^{\epsilon} - 1}{2\delta} + 1\right).$ Notice that $\delta \leq \frac{1}{2}$ implies $t \geq 1$. 
 	For each vector $\vec{d} \in \Z^m$, let $D_{\vec{d}}$ be a database where for any $\vec{d}, \vec{d}' \in \Z^m$, if $\norm{\vec{d} - \vec{d}'}_1 \leq 1$, then $D_{\vec{d}}$ and $D_{\vec{d}'}$ are neighboring. Let $\vec{b}\left(D_{\vec{d}}\right) = \Delta \vec{d}$ and let $a_1, \dots, a_m > 0$ be the diagonal entries of $\A$. Since $\A\vec{\mu}\left(D_{\vec{d}}\right) \leq \vec{b}\left(D_{\vec{d}}\right)$ with probability 1, $\vec{\mu}\left(D_{\vec{d}}\right)$ must be coordinate-wise smaller than $\Delta \left(\frac{d_1}{a_1}, \dots, \frac{d_m}{a_m}\right)$.
 
 We begin by partitioning the support of $\vec{\mu}\left(D_{\vec{d}}\right)$ so that we can analyze $\E\left[g\left(\vec{\mu}\left(D_{\vec{d}}\right)\right)\right]$ using the law of total expectation. We organize this partition using axis-aligned rectangles.
Specifically, for each index $i \in [m]$, let $S_{i}^0$ be the set of vectors $\vec{x} \in \R^m$ whose $i^{th}$ components are smaller than $\frac{\Delta}{a_i}\left(d_i - \lfloor t \rfloor\right)$: \[S_{i}^0 = \left\{\vec{x} \in \R^m : x_i \leq \frac{\Delta}{a_i}\left(d_i - \lfloor t \rfloor\right)\right\}.\] Similarly, let \[S_{i}^1 = \left\{\vec{x} \in \R^m : \frac{\Delta}{a_i}\left(d_i - \lfloor t \rfloor\right) < x_i \leq \frac{\Delta d_i}{a_i}\right\}.\] For any vector $\vec{I} \in \{0,1\}^m$, let $S_{\vec{I}} = \cap_{i = 1}^m S_i^{I_i}$. The sets $S_{\vec{I}}$ partition the support of $\vec{\mu}\left(D_{\vec{d}}\right)$ into rectangles. Therefore, by the law of total expectation, \begin{equation}\E\left[g\left(\vec{\mu}\left(D_{\vec{d}}\right)\right)\right] = \sum_{{\vec{I}} \in \{0,1\}^m}\E\left[g\left(\vec{\mu}\left(D_{\vec{d}}\right)\right) \mid \vec{\mu}\left(D_{\vec{d}}\right) \in S_{{\vec{I}}}\right]\mathbb{P}\left[\vec{\mu}\left(D_{\vec{d}}\right) \in S_{{\vec{I}}}\right].\label{eq:law_app}\end{equation}
 
 Conditioning the vector $\vec{\mu}\left(D_{\vec{d}}\right)$ to lie within a rectangle $S_{{\vec{I}}}$ makes it much easier to analyze the expected value of $g\left(\vec{\mu}\left(D_{\vec{d}}\right)\right)$.
 Suppose that $\vec{\mu}\left(D_{\vec{d}}\right) \in S_{{\vec{I}}}$ for some $\vec{I} \in \{0,1\}^m$. If $I_i = 0$, then we know that $\mu\left(D_{\vec{d}}\right)_i \leq \frac{\Delta}{a_i}\left(d_i - \lfloor t \rfloor\right)$. Meanwhile, if $I_i = 1$, then $\mu\left(D_{\vec{d}}\right)_i \leq \frac{\Delta d_i}{a_i}$ since $\A \vec{\mu}\left(D_{\vec{d}}\right) \leq \vec{b}\left(D_{\vec{d}}\right)$ with probability 1. Since $g(\vec{x}) = \langle \vec{1}, \vec{x} \rangle$, we have that for each $\vec{I} \in \{0,1\}^m$,
\[E\left[g\left(\vec{\mu}\left(D_{\vec{d}}\right)\right) \mid \vec{\mu}\left(D_{\vec{d}}\right) \in S_{{\vec{I}}}\right]\leq\sum_{i = 1}^m \frac{\Delta\left(d_i - \lfloor t \rfloor\right)}{a_i}\textbf{1}_{\{I_i = 0\}} + \frac{\Delta d_i}{a_i}\textbf{1}_{\{I_i = 1\}} = \sum_{i = 1}^m \frac{\Delta d_i}{a_i}-  \frac{\Delta\lfloor t \rfloor}{a_i}  \textbf{1}_{\{I_i = 0\}}.\]
 Combining this inequality with Equation~\eqref{eq:law_app} and rearranging terms,
 we have that \begin{align*} \E\left[g\left(\vec{\mu}\left(D_{\vec{d}}\right)\right)\right] &\leq  \Delta\sum_{i = 1}^m \frac{d_i}{a_i}-  \sum_{{\vec{I}} \in \{0,1\}^m}\sum_{i = 1}^m \frac{\Delta\lfloor t \rfloor}{a_i}  \textbf{1}_{\{I_i = 0\}}\Pr\left[\vec{\mu}\left(D_{\vec{d}}\right) \in S_{{\vec{I}}}\right]\\
&= \Delta\sum_{i = 1}^m \frac{d_i}{a_i} -  \Delta\lfloor t \rfloor \sum_{i = 1}^m\frac{1}{a_i}\sum_{{\vec{I}} \in \{0,1\}^m} \textbf{1}_{\{I_i = 0\}} \Pr\left[\vec{\mu}\left(D_{\vec{d}}\right) \in S_{{\vec{I}}}\right].\end{align*} For any $i \in [m]$, $\sum_{{\vec{I}} \in \{0,1\}^m} \textbf{1}_{\{I_i = 0\}} \Pr\left[\vec{\mu}\left(D_{\vec{d}}\right) \in S_{{\vec{I}}}\right] = \Pr\left[\vec{\mu}\left(D_{\vec{d}}\right) \in S_i^0\right].$ Therefore,
\begin{equation}\E\left[g\left(\vec{\mu}\left(D_{\vec{d}}\right)\right)\right]\leq \Delta\sum_{i = 1}^m \frac{d_i}{a_i} - \Delta \lfloor t \rfloor \sum_{i = 1}^m\frac{1}{a_i}\Pr\left[\vec{\mu}\left(D_{\vec{d}}\right) \in S_i^0\right].\label{eq:bound_exp_app}\end{equation}

We now prove that for every index $i \in [m]$, $\Pr\left[\vec{\mu}\left(D_{\vec{d}}\right) \in S_i^0\right] > \frac{1}{2}$. This proof relies on the following claim.
 	
 	\begin{restatable}{claim}{induction}\label{claim:md_induction}
 		For any index $i \in [m]$, vector $\bar{\vec{d}} \in \Z^m$, and integer $j \geq 1$, let $S_{\bar{\vec{d}},i, j}$ be the set of all vectors $\vec{x} \in \R^m$ whose $i^{th}$ component is in the interval $\left(\frac{\Delta(\bar{d}_i - j)}{a_i},  \frac{\Delta \bar{d}_i}{a_i}\right]$: \[S_{\bar{\vec{d}},i, j} = \left\{\vec{x} \in \R^m : \frac{\Delta(\bar{d}_i - j)}{a_i} < x_i \leq \frac{\Delta \bar{d}_i}{a_i}\right\}.\]
Then $\Pr\left[\vec{\mu}\left(D_{\bar{\vec{d}}}\right) \in S_{\bar{\vec{d}},i, j}\right] \leq \delta \sum_{\ell =0}^{j-1} e^{\epsilon \ell}$.
 	\end{restatable}
 	Notice that $S_{\vec{d},i, \lfloor t \rfloor} = S_i^1$, a fact that will allow us to prove that $\Pr\left[\vec{\mu}\left(D_{\vec{d}}\right) \in S_i^0\right] > \frac{1}{2}$.
   	\begin{proof}[Proof of Claim~\ref{claim:md_induction}]
 		We prove this claim by induction on $j$.
 		\paragraph{Base case $(j=1)$.} Fix an arbitrary index $i \in [m]$ and vector $\bar{\vec{d}} \in \Z^m$. Let $\vec{e}_i \in \{0,1\}^m$ be the standard basis vector with a 1 in the $i^{th}$ component and 0 in every other component. Since $\vec{b}\left(D_{\bar{\vec{d}} - \vec{e}_i}\right) = \Delta\left(\bar{\vec{d}} - \vec{e}_i\right)$, we know the probability that $\mu\left(D_{\bar{\vec{d}}- \vec{e}_i}\right)_i > \frac{\Delta\left(\bar{d}_i - 1\right)}{a_i}$  is zero. In other words, \begin{equation}\Pr\left[\vec{\mu}\left(D_{\bar{\vec{d}}- \vec{e}_i}\right) \in S_{\bar{\vec{d}},i,1}\right] = 0.\label{eq:md_base_case}\end{equation} Since $D_{\bar{\vec{d}}}$ and $D_{\bar{\vec{d}}- \vec{e}_i}$ are neighboring, this means that \[\Pr\left[\vec{\mu}\left(D_{\bar{\vec{d}}}\right) \in S_{\bar{\vec{d}},i,1}\right]\leq e^{\epsilon}\Pr\left[\vec{\mu}\left(D_{\bar{\vec{d}}- \vec{e}_i}\right) \in S_{\bar{\vec{d}},i,1}\right] + \delta = \delta.\]
 		
 		\paragraph{Inductive step.} Fix an arbitrary $j \geq 1$ and suppose that for all indices $i \in [m]$ and vectors $\bar{\vec{d}} \in \Z^m$, $\Pr\left[\vec{\mu}\left(D_{\bar{\vec{d}}}\right) \in S_{\bar{\vec{d}},i,j}\right] \leq \delta \sum_{\ell =0}^{j-1} e^{\epsilon \ell}.$ We want to prove that for all indices $i \in [m]$ and vectors $\bar{\vec{d}} \in \Z^m$, $\Pr\left[\vec{\mu}\left(D_{\bar{\vec{d}}}\right) \in S_{\bar{\vec{d}},i,j+1}\right] \leq \delta \sum_{\ell =0}^{j} e^{\epsilon \ell}.$ To this end, fix an arbitrary index $i \in [m]$ and vector $\bar{\vec{d}} \in \Z^m$. By the inductive hypothesis, we know that \begin{align}\Pr\left[\vec{\mu}\left(D_{\bar{\vec{d}} -\vec{e}_i }\right) \in S_{\bar{\vec{d}} -\vec{e}_i ,i,j}\right]\label{eq:md_ind_hyp_app}
 		\leq \delta \sum_{\ell =0}^{j-1} e^{\epsilon \ell}.\end{align} Note that \begin{align*}S_{\bar{\vec{d}},i, j + 1}&= \left\{\vec{x} : \frac{\Delta\left(\bar{d}_i - j - 1\right)}{a_i} < x_i \leq \frac{\Delta \bar{d}_i}{a_i}\right\}\\
 		&= \left\{\vec{x} : \frac{\Delta\left(\bar{d}_i - j - 1\right)}{a_i} < x_i \leq \frac{\Delta (\bar{d}_i - 1)}{a_i}\right\} \cup \left\{\vec{x} : \frac{\Delta\left(\bar{d}_i - 1\right) }{a_i}< x_i \leq \frac{\Delta \bar{d}_i}{a_i}\right\}\\
 		&= S_{\bar{\vec{d}} - \vec{e}_i,i,j} \cup S_{\bar{\vec{d}},i, 1}.\end{align*}
 		
 		We can now use this fact, the definition of differential privacy, and Equation~\eqref{eq:md_ind_hyp_app}, to prove the inductive hypothesis holds. By the definition of differential privacy, \[\Pr\left[\vec{\mu}\left(D_{\bar{\vec{d}}}\right) \in S_{\bar{\vec{d}},i, j+1}\right] \leq e^{\epsilon}\Pr\left[\vec{\mu}\left(D_{\bar{\vec{d}}-\vec{e}_i}\right) \in S_{\bar{\vec{d}},i,j+1}\right] + \delta.\] Since $S_{\bar{\vec{d}},i, j + 1} = S_{\bar{\vec{d}} - \vec{e}_i,i,j} \cup S_{\bar{\vec{d}},i, 1}$, we have that
 		\[\Pr\left[\vec{\mu}\left(D_{\bar{\vec{d}}}\right) \in S_{\bar{\vec{d}},i, j+1}\right]\leq e^{\epsilon}\left(\Pr\left[\vec{\mu}\left(D_{\bar{\vec{d}} - \vec{e}_i}\right) \in S_{\bar{\vec{d}} - \vec{e}_i,i,j}\right] + \Pr\left[\vec{\mu}\left(D_{\bar{\vec{d}} - \vec{e}_i}\right) \in S_{\bar{\vec{d}},i,1}\right]\right) + \delta.\]
 		By Equation~\eqref{eq:md_base_case}, we know that $\Pr\left[\vec{\mu}\left(D_{\bar{\vec{d}} - \vec{e}_i}\right) \in S_{\bar{\vec{d}},i,1}\right] = 0$, so \[\Pr\left[\vec{\mu}\left(D_{\bar{\vec{d}}}\right) \in S_{\bar{\vec{d}},i, j+1}\right] \leq e^{\epsilon}\Pr\left[\vec{\mu}\left(D_{\bar{\vec{d}} - \vec{e}_i}\right) \in S_{\bar{\vec{d}} - \vec{e}_i,i,j}\right] + \delta.\] Finally, by the inductive hypothesis (Equation~\eqref{eq:md_ind_hyp_app}), \[\Pr\left[\vec{\mu}\left(D_{\bar{\vec{d}}}\right) \in S_{\bar{\vec{d}},i, j+1}\right] \leq e^{\epsilon}\delta \sum_{\ell =0}^{j-1} e^{\epsilon \ell} + \delta = \delta \sum_{\ell =0}^{j} e^{\epsilon \ell},\]
 		so the inductive hypothesis holds.
 	\end{proof}
 
Since $S_{\vec{d},i, \lfloor t \rfloor} = S_i^1$,
our careful choice of the value $t$ allows us to prove that for every index $i$, $\Pr\left[\vec{\mu}\left(D_{\vec{d}}\right) \in S_i^0\right] > \frac{1}{2}$ (see Claim~\ref{claim:md_lb} in Appendix~\ref{app:multi_dim}).
 	This inequality, Equation~\eqref{eq:bound_exp_app}, and the fact that $t \geq 1$ together imply that \[
 	\E\left[g\left(\vec{\mu}\left(D_{\vec{d}}\right)\right)\right] < \Delta\sum_{i = 1}^m \frac{d_i}{a_i} - \frac{\Delta \lfloor t \rfloor }{2}\sum_{i = 1}^m\frac{1}{a_i} \leq \Delta\sum_{i = 1}^m \frac{d_i}{a_i} - \frac{\Delta t}{4}\sum_{i = 1}^m\frac{1}{a_i}.\] Since $\max\left\{g(\vec{x}) : \A \vec{x}\leq\vec{b}\left(D_{\vec{d}}\right) \right\} = \max\left\{\langle \vec{1}, \vec{x}\rangle : \vec{x}\leq \Delta \left(\frac{d_1}{a_1}, \dots, \frac{d_m}{a_m}\right)\right\} = \Delta \sum_{i= 1}^m \frac{d_i}{a_i},$ we have that
 	\[\max\left\{g(\vec{x}) : \A \vec{x}\leq\vec{b}\left(D_{\vec{d}}\right) \right\}  - \E\left[g\left(\vec{\vec{\mu}}\left(D_{\vec{d}}\right)\right)\right]
 	\geq  \frac{\Delta }{4\epsilon}\left(\sum_{i = 1}^m \frac{1}{a_i}\right) \ln \left(\frac{e^{\epsilon} - 1}{2\delta} + 1\right).\]
 	
We now prove that $\inf_{p \geq 1}\alpha_{p,1}(\A) \sqrt[p]{m} \leq \alpha_{\infty,1}(\A)  = \sum_{i = 1}^m \frac{1}{a_i}$, which proves the theorem statement.
 	Since $\A$ is diagonal,  $\alpha_{\infty,1}(\A) = \sup_{\vec{u} \geq \vec{0}}\left\{\norm{\vec{u}}_{1} : u_{i}a_{i} \leq 1, \forall i \in [m]\right\} = \sum_{i = 1}^m \frac{1}{a_i}.$  Moreover, since $\sqrt[\infty]{m} = 1$, $\alpha_{\infty,1}(\A) \in \left\{\alpha_{p,1}(\A) \sqrt[p]{m} : p \geq 1\right\}$, which implies that $\inf_{p \geq 1}\alpha_{p,1}(\A) \sqrt[p]{m} \leq \alpha_{\infty,1}(\A).$ Therefore,  	\[\max\left\{g(\vec{x}) : \A \vec{x}\leq\vec{b}(D_i) \right\}  - \E[g(\vec{\vec{\mu}}(D))]\geq \frac{\Delta}{4\epsilon} \cdot \inf_{p \geq 1} \left\{\alpha_{p,1}(\A)\sqrt[p]{m} \right\} \cdot \ln \left(\frac{e^{\epsilon} - 1}{2\delta} + 1\right),\] as claimed.
 \end{proof}

    \begin{claim}\label{claim:md_lb}
    	Let $D_{\vec{d}}$ and $S_i^0$ be  defined as in Theorem~\ref{thm:lb}. Then $\Pr\left[\vec{\mu}(D_{\vec{d}}) \in S_i^0\right] > \frac{1}{2}$.
    \end{claim}
    
    \begin{proof}
    Let $S_i^1$ and $S_{\vec{d},i, \lfloor t \rfloor}$ be  defined as in Theorem~\ref{thm:lb}.
We prove that $\Pr\left[\mu(D_{\vec{d}}) \in S_{\vec{d},i, \lfloor t \rfloor}\right] \leq \frac{1}{2}$. Since $S_{\vec{d},i, \lfloor t \rfloor} = S_i^1$, this implies that $\Pr\left[\mu(D_{\vec{d}}) \in S_i^0\right] > \frac{1}{2}$. This claim follows from the following chain of inequalities, which themselves follow from Claim~\ref{claim:md_induction} and the fact that $t = \frac{1}{\epsilon} \log \left(\frac{e^{\epsilon} - 1}{2\delta} + 1\right)$:
    	\begin{align*}
    	\Pr\left[\mu(D_{\vec{d}}) \in S_{\vec{d},i, \lfloor t \rfloor}\right] &\leq \delta \sum_{j = 0}^{\lfloor t \rfloor - 1} e^{\epsilon j}\\
    	&= \frac{\delta\left(e^{\epsilon\lfloor t \rfloor} - 1\right)}{e^\epsilon - 1}\\
    	&\leq \frac{\delta\left(\exp\left(\epsilon t\right) - 1\right)}{\exp(\epsilon) - 1}\\
    	&= \frac{\delta\left(\exp\left(\epsilon\left(\frac{1}{\epsilon} \log \left(\frac{e^{\epsilon} - 1}{2\delta} + 1\right)\right)\right) - 1\right)}{\exp(\epsilon) - 1}\\
    	&= \frac{\delta\left(\exp\left(\log \left(\frac{e^{\epsilon} - 1}{2\delta} + 1\right)\right) - 1\right)}{\exp(\epsilon) - 1}\\
    	&= \frac{\delta\left(\left(\frac{e^{\epsilon} - 1}{2\delta} + 1\right) - 1\right)}{\exp(\epsilon) - 1}\\
    	&= \frac{\delta\left(\frac{e^{\epsilon} - 1}{2\delta}\right)}{\exp(\epsilon) - 1}\\
    	&= \frac{1}{2}.
    	\end{align*}
    	Therefore, $\Pr\left[\vec{\mu}(D_{\vec{d}}) \in S_i^0\right] > \frac{1}{2}$.
    \end{proof}

\section{Characterizing $(\epsilon,0)$-differentially private mechanisms}\label{app:characterization}
In Section~\ref{sec:multi_dim} we presented a nearly optimal $(\epsilon, \delta)$-DP mechanism for private optimization. The optimal $(\epsilon, 0)$-DP mechanism for this problem is considerably easier to characterize.

\begin{theorem}\label{thm:pure_DP}
Let $S^* = \bigcap_{D \subseteq \cX}\left\{\vec{x} : \A \vec{x} \leq \vec{b}(D)\right\}$ be the intersection of all feasible sets across all databases $D$. If $S^*$ is nonempty, then the optimal $(\epsilon, 0)$-differentially private mechanism outputs $\argmax_{\vec{x} \in S^*} g(\vec{x})$  with probability 1. If $S^*$ is empty, then no $(\epsilon, 0)$-differentially private mechanism exists.
\end{theorem}

\begin{proof}
Fix a mechanism, and let $P(D)$ be the set of vectors $\vec{x}$ in the support of the mechanism's output given as input the database $D$. We claim that if the mechanism is $(\epsilon, 0)$-differentially private, then there exists a set $P^*$ such that $P(D) = P^*$ for all databases $D$. Suppose, for the sake of a contradiction, that there exist databases $D$ and $D'$ such that $P(D) \not= P(D')$. Let $D_1, \dots, D_n$ be a sequence of databases such that $D_1 = D$, $D_n = D'$, and each pair of databases $D_i$ and $D_{i+1}$ are neighbors. Then there must exist a pair of neighboring databases $D_i$ and $D_{i+1}$ such that $P(D_i) \not= P(D_{i+1})$, which contradicts the fact that the mechanism is $(\epsilon, 0)$-differentially private.
Therefore, if the mechanism is $(\epsilon, 0)$-differentially private, then to satisfy the feasibility requirement, we must have that $P^* \subseteq S^*$. If $S^*$ is empty, then no such mechanism exists. If $S^*$ is nonempty, then the optimal $(\epsilon, 0)$-differentially private mechanism outputs $\argmax_{\vec{x} \in S^*}g(\vec{x})$ with probability 1.
\end{proof}

\section{Additional information about experiments}\label{app:experiments}
In Figure~\ref{fig:quality_app}, we analyze the quality of our algorithm for several different parameter settings. First, we select the number of individuals $n$ to be a value in $\{500,1000, 1500\}$ (the number of investors is $n = 500$ in Figures~\ref{fig:n500r11}-\ref{fig:n500r13}, $n = 1000$ in Figures~\ref{fig:n1000r22}-\ref{fig:n1000r27}, and $n = 1500$ in Figures~\ref{fig:n1500r33}-\ref{fig:n1500r42}). Then, we define each element of the database (money given by individuals to an investor) as a draw from the uniform distribution between 0 and 1, so $b(D)$ equals the sum of these $n$ random variables. The sensitivity of $b(D)$ is therefore $\Delta = 1$. We set the minimum return $r_{min}$ to be a value in $[1,5]$.
We calculate the objective value $v^* \in \R$ of the optimal solution to Equation~\eqref{eq:portfolio}. Then, for $\delta \in \left[\frac{1}{n^2}, 0.002\right]$ and $\epsilon \in [0.5, 2.5]$, we run our algorithm 50 times and calculate the average objective value $\hat{v}_{\epsilon, \delta} \in \R$ of the optimal solutions.
\begin{figure}[t]
	\centering
	\begin{subfigure}{.32\textwidth}
		\includegraphics[width=\textwidth]{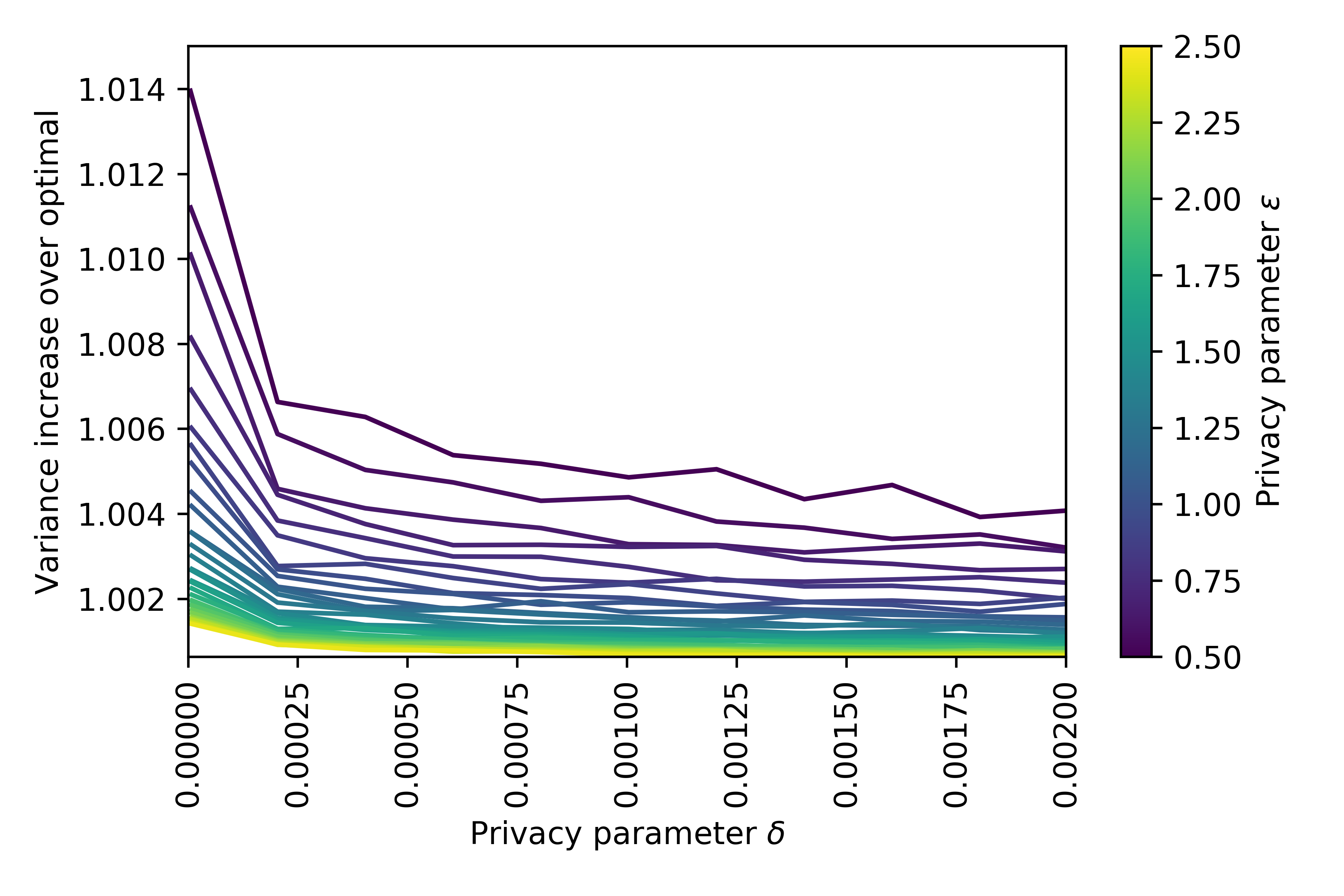}\centering
		\caption{$n = 500$ and $r_{min} = 1.1$}\label{fig:n500r11}
	\end{subfigure}
	\begin{subfigure}{.32\textwidth}
		\includegraphics[width=\textwidth]{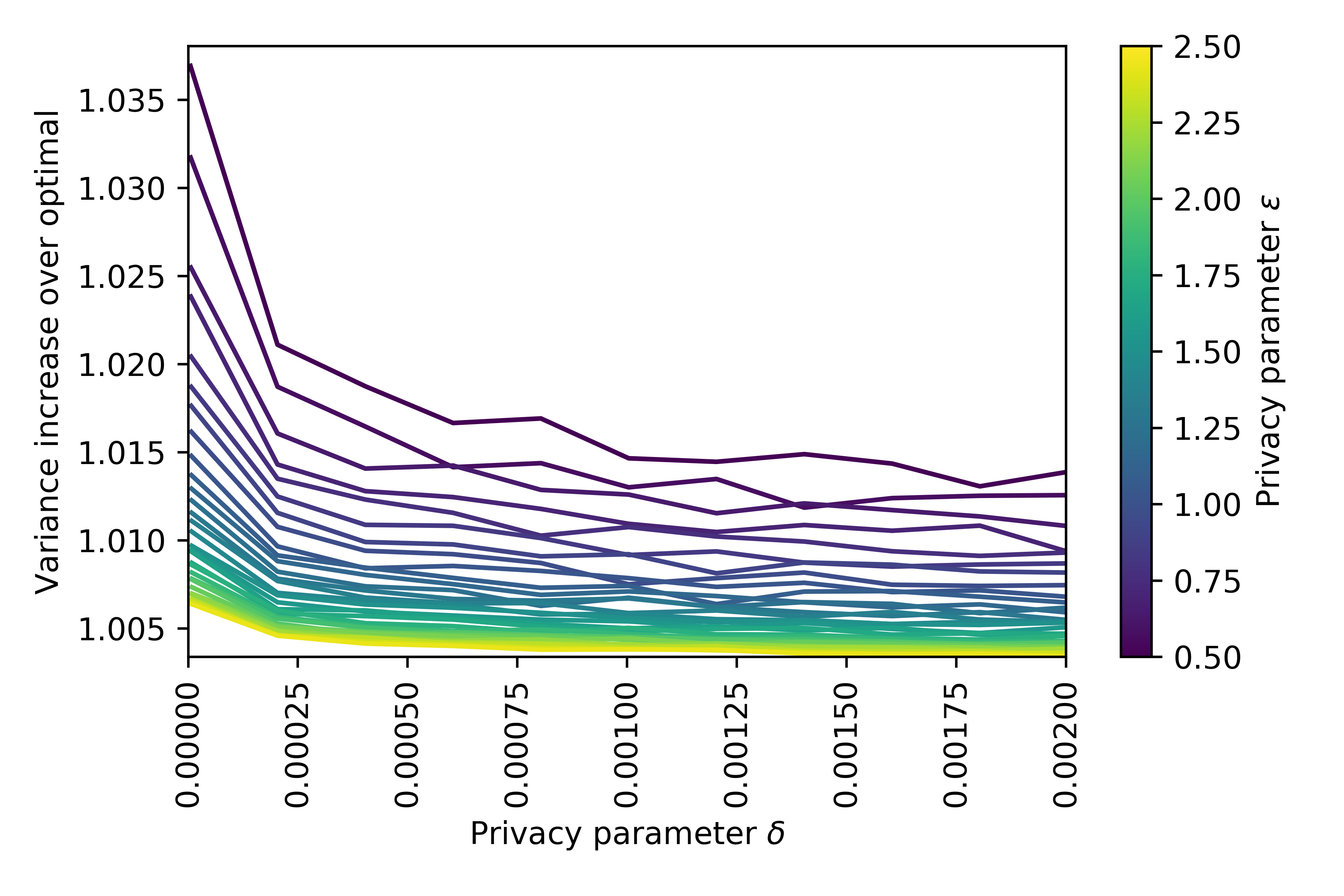}\centering
		\caption{$n = 500$ and $r_{min} = 1.2$}\label{fig:n500r12}
	\end{subfigure}
	\begin{subfigure}{.32\textwidth}
		\includegraphics[width=\textwidth]{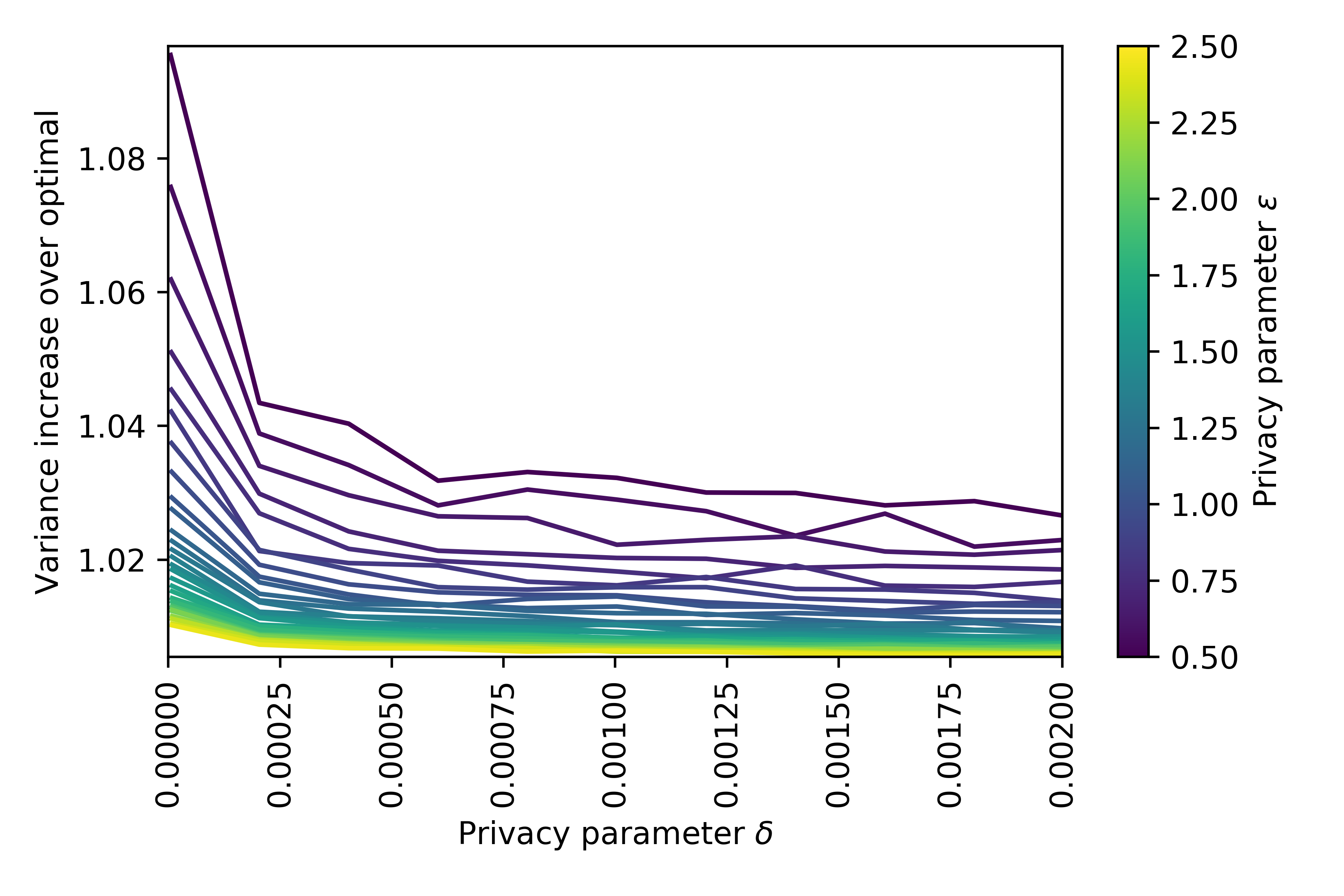}\centering
		\caption{$n = 500$ and $r_{min} = 1.3$}\label{fig:n500r13}
	\end{subfigure}
	\begin{subfigure}{.32\textwidth}
		\includegraphics[width=\textwidth]{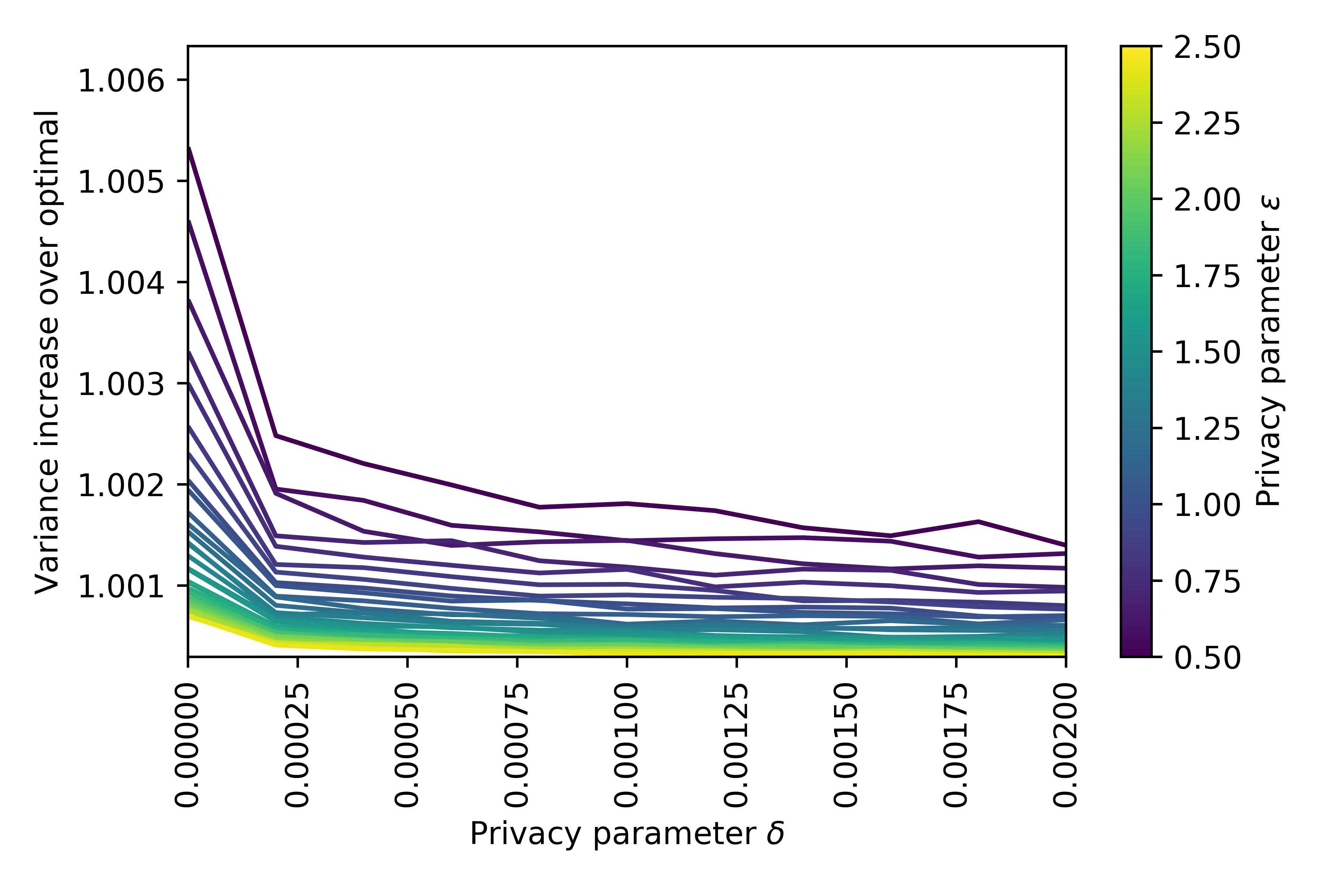}\centering
		\caption{$n = 1000$ and $r_{min} = 2.2$}\label{fig:n1000r22}
	\end{subfigure}
	\begin{subfigure}{.32\textwidth}
		\includegraphics[width=\textwidth]{n1000_r25}\centering
		\caption{$n = 1000$ and $r_{min} = 2.5$}\label{fig:n1000r25}
	\end{subfigure}
	\begin{subfigure}{.32\textwidth}
		\includegraphics[width=\textwidth]{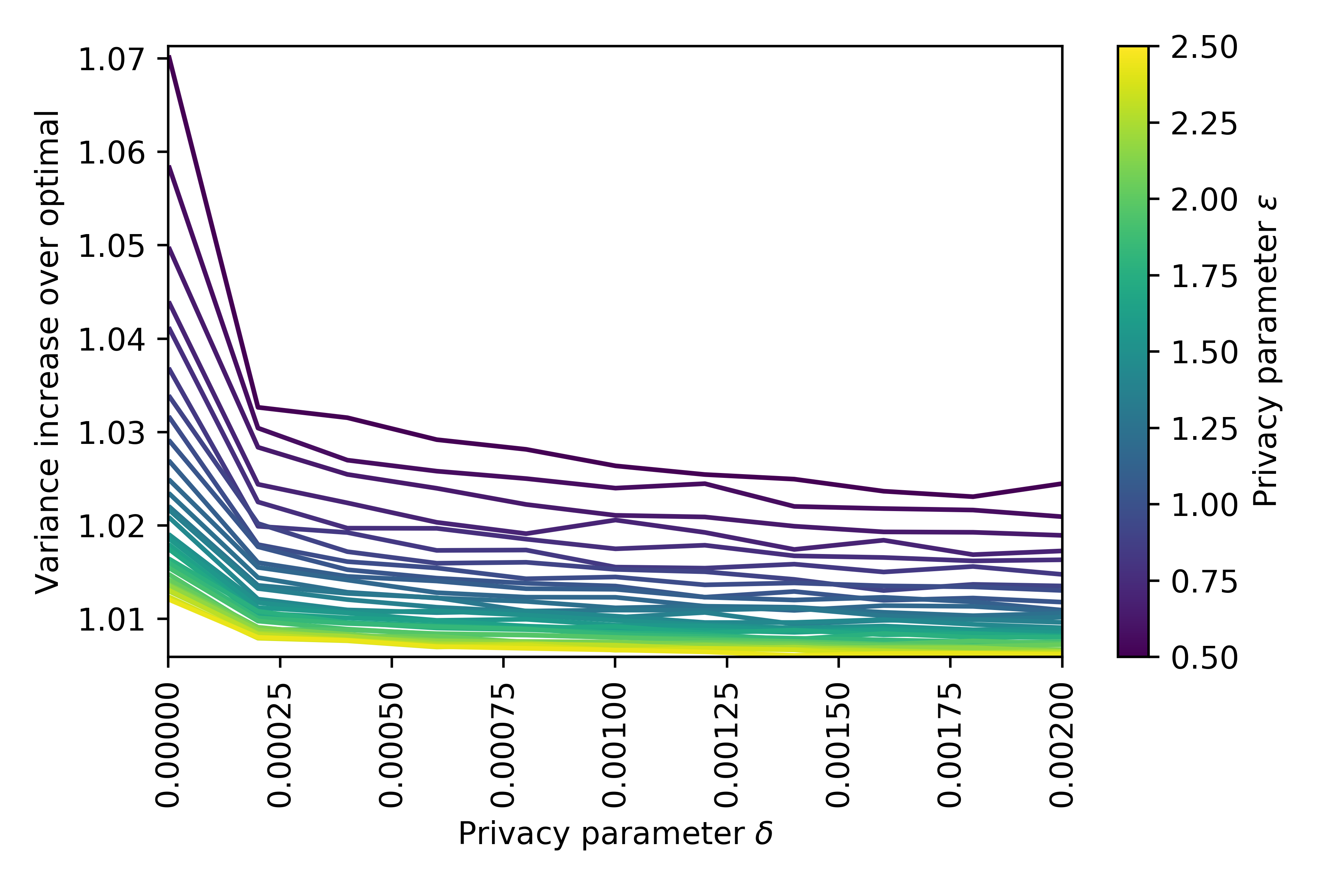}\centering
		\caption{$n = 1000$ and $r_{min} = 2.7$}\label{fig:n1000r27}
	\end{subfigure}
	\begin{subfigure}{.32\textwidth}
		\includegraphics[width=\textwidth]{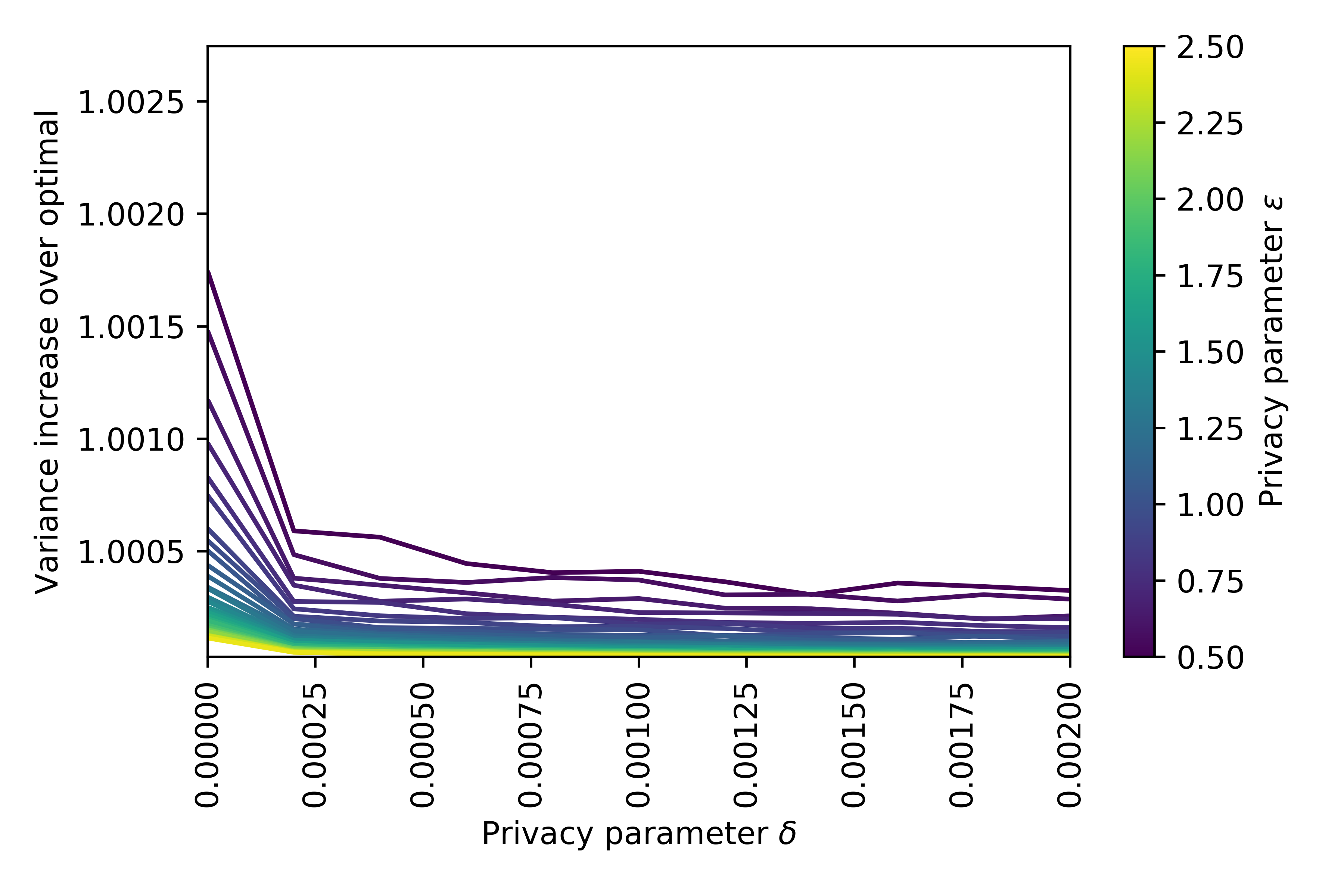}\centering
		\caption{$n = 1500$ and $r_{min} = 3.3$}\label{fig:n1500r33}
	\end{subfigure}
	\begin{subfigure}{.32\textwidth}
		\includegraphics[width=\textwidth]{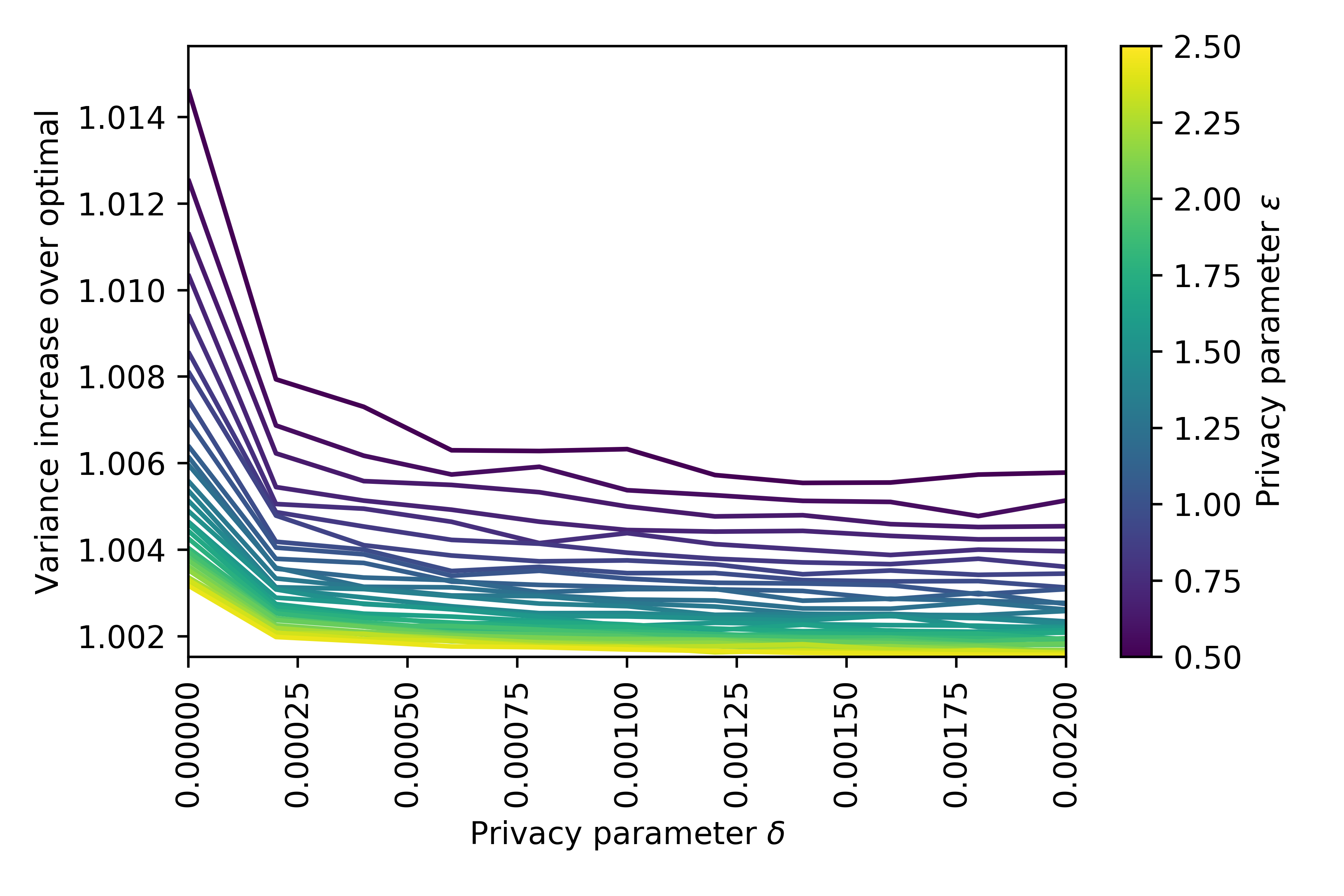}\centering
		\caption{$n = 1500$ and $r_{min} = 3.8$}\label{fig:n1500r38}
	\end{subfigure}
	\begin{subfigure}{.32\textwidth}
		\includegraphics[width=\textwidth]{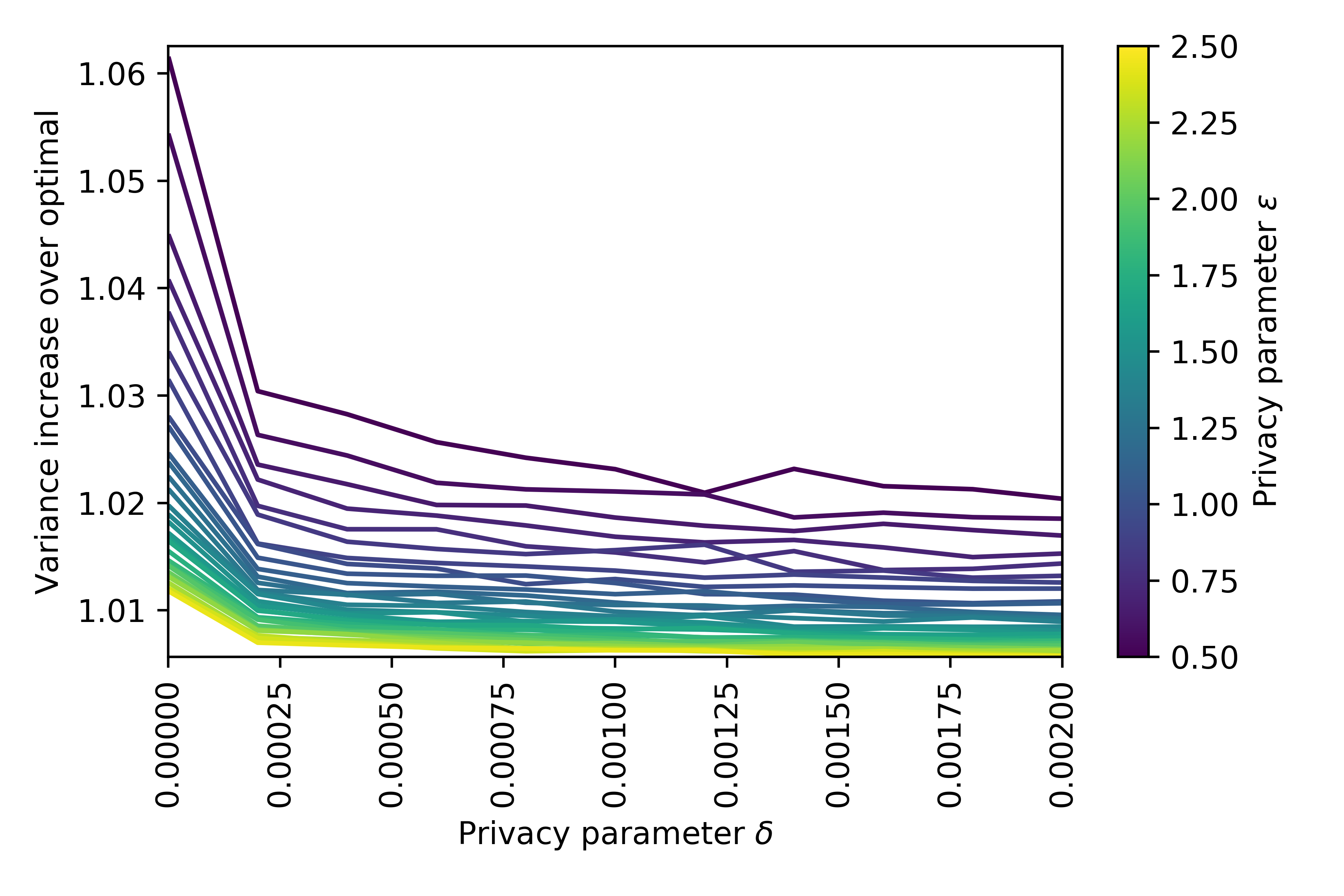}\centering
		\caption{$n = 1500$ and $r_{min} = 4.2$}\label{fig:n1500r42}
	\end{subfigure}
	\caption{Quality in the portfolio optimization application for various choices of the number $n$ of investors and the minimum return $r_{min}$.  These plots show the multiplicative increase in the objective function value of our algorithm's solution---for various choices of $\epsilon$ and $\delta$---over the objective function value of the optimal solution to the original optimization problem (Equation~\eqref{eq:portfolio}). Darker shading corresponds to lower values of $\epsilon$ and therefore stronger privacy.}
	\label{fig:quality_app}
\end{figure}

We find that there is a sweet spot for the parameter choices $n$ and $r_{min}$. If $r_{min}$ is too small, the budget constraint is non-binding with or without privacy, so the variance increase over optimal is always 1.
We find this is true when $n = 500$ and $r_{min} \leq 1$, when $n = 1000$ and $r_{min} \leq 2.1$, and when $n = 1500$ and $r_{min} \leq 3.2$. This also explains why the variance increase over optimal improves as $r_{min}$ shrinks, as we can observe from Figure~\ref{fig:quality_app}.
Meanwhile, if $r_{min}$ is too large, then the original quadratic program (Equation~\eqref{eq:portfolio}) is infeasible. We find this is true when $n = 500$ and $r_{min} \geq 1.4$, when $n = 1000$ and $r_{min} \geq 2.8$, and when $n = 1500$ and $r_{min} \geq 4.3$. 

\end{document}